\documentclass[sn-mathphys]{sn-jnl}% Math and Physical Sciences Reference Style
\usepackage{bm}
\usepackage{comment}

\jyear{2021}%

\theoremstyle{thmstyleone}%
\newtheorem{theorem}{Theorem}%  meant for continuous numbers
\theoremstyle{thmstyletwo}%

\theoremstyle{thmstylethree}%

\newtheorem{lemma}{Lemma}
\newtheorem{corollary}{Corollary}

\begin{document}

\title{Chance constrained conic-segmentation support vector machine with uncertain data}


\author*[1,2]{\fnm{Shen} \sur{Peng}}\email{pengshen@xidian.edu.cn}

\author[2]{\fnm{Gianpiero} \sur{Canessa}}\email{canessa@kth.se}

\author[1]{\fnm{Zhihua} \sur{Allen-Zhao}}\email{allenzhaozh@gmail.com}

%\author*[1,2]{\fnm{First} \sur{Author}}\email{iauthor@gmail.com}
\affil*[1]{\orgdiv{School of Mathematics and Statistics}, \orgname{Xidian University}, \orgaddress{\state{Xi’an}, \postcode{710126}, \country{China}}}
\affil[2]{\orgdiv{Department of Mathematics}, \orgname{KTH Royal Institute of Technology}, \orgaddress{\street{Lindstedtsv\"agen 25}, \postcode{SE-100 44}, \state{Stockholm}, \country{Sweden}}}

%\affil[2]{\orgdiv{Department}, \orgname{Organization}, \orgaddress{\street{Street}, \city{City}, \postcode{10587}, \state{State}, \country{Country}}}
%
%\affil[3]{\orgdiv{Department}, \orgname{Organization}, \orgaddress{\street{Street}, \city{City}, \postcode{610101}, \state{State}, \country{Country}}}


\abstract{Support vector machines (SVM) is one of the well known supervised machine learning model.
The standard SVM models are dealing with the situation where the exact values of the data points are known.
This paper studies the SVM model when the data set contains uncertain or mislabelled data points.
To ensure the small probability of misclassification for the uncertain data, a chance constrained conic-segmentation SVM model is proposed for multiclass classification.
Based on the data set, a mixed integer programming formulation for the chance constrained conic-segmentation SVM is derived.
Kernelization of chance constrained conic-segmentation SVM model is also exploited for nonlinear classification.
The geometric interpretation is presented to show how the chance constrained conic-segmentation SVM works on uncertain data.
Finally, experimental results are presented to demonstrate the effectiveness of the chance constrained conic-segmentation SVM for both artificial and real-world data.}




\maketitle

\section{Introduction}

In classification problems, a classifier is a function that mimics the relationship between the data vectors and their class labels.
Support vector machine (SVM) is a popular classifier, which was proposed by Cortes and Vapnik \cite{Vapnik1995} as a maximum margin classifier.
The success of the SVM has encouraged further research into extensions to the more general multiclass cases, which has been an active topic of research interest \cite{Mu2009,Angulo2006multi,Tian2012recent}.
Shilton et al.\cite{Shilton2012} proposed the conic-segmentation support vector machine (CS-SVM) by introducing the concept of target space into the
problem formulation and showed that some other multiclassfication model are special cases of this framework.

The standard CS-SVM is dealing with the situation where the exact values of the data points are known.
When the data points are uncertain or mislabelled, different robust models have been proposed to formulate the SVM with uncertainties \cite{Xanthopoulos2012book,Xanthopoulos2014robust,Fan2014}.
Xanthopoulos et al. \cite{Xanthopoulos2012book} proposed a robust optimization model where the
perturbation of the uncertain data is bounded by norm.
Xanthopoulos et al. \cite{Xanthopoulos2014robust} considered a robust
generalized eigenvalue classifier with ellipsoidal uncertainty of the data set.
Fan et al. \cite{Fan2014} derived
a robust model for the polyhedral uncertainties.
Stochastic programming is a natural approach to deal with uncertainty.
Models with chance constraints are used to ensure the small probability of misclassification for the uncertain data \cite{Shivaswamy2006second,Ben2011chance,Wang2018robust}.
%{\color{blue}
To deal with the chance constraint,
Shivaswamy et al. \cite{Shivaswamy2006second} and Ben-
Tal et al. \cite{Ben2011chance} transformed the chance constraint by Chebyshev inequality and Bernstein bounding schemes, respectively.
While, Wang et al. \cite{Wang2018robust} proposed the distributionally robust chance constraints for each uncertain data point with a moments based ambiguity set.
%}

Chance constrained optimization problem was first introduced and studied by Charnes et al. \cite{Charnes1959} and Miller and Wagner \cite{Miller1965}.
Since then, chance constrained optimization has been studied extensively in the stochastic programming literature.
However, this problem is difficult to solve in general.
One main reason is that the probability in the constraint generally has no closed form, often is non-convex and is typically difficult to compute.

When the samples of the random variable are available, the sample average approximation (SAA) approach can be applied to solve the chance constrained optimization problem approximately.
Such an approximation is obtained by replacing the actual distribution in chance constraint by an empirical distribution corresponding to a sample.
SAA methods for chance constrained optimization problems have been investigated in \cite{Luedtke2008,Pagnoncelli2009,Luedtke2010}.

To handle the uncertainty and mislabeling in the data set, we propose a stochastic approach to CS-SVM by introducing chance constraints into the problem formulation.
Based on the samples, the chance constrained CS-SVM model is derived as a mixed integer programming problem using SAA.
For nonlinear classification, the kernelization of the chance constrained CS-SVM model is proposed.
We illustrate the sample based reformulation for the chance constrained CS-SVM geometrically.
And an experiment is proposed to illustrate the performance of the chance constrained CS-SVM model.

%{\color{blue}
Different from the existing work about the chance constrained support vector machine, our work bases on conic-segmentation support vector machine for multiclassification, which can encompasses some well-known multiclass
methods \cite{Shilton2012}.
In addition, the proposed chance constrained CS-SVM is solved by SAA method based on the data set, but not probability inequalities, which can be conservative in practice.
Compared with the work in \cite{Wang2018robust}, which constructed the chance constraint for each data point, the chance constraint in this work is modeled for each class.
%}

The paper is organized as follows.
We present the background and modelling of chance constrained CS-SVM in section 2.
In section 3, we derive the sample based reformulation and kernelization of chance constrained CS-SVM and illustrate it geometrically.
In section 4, we give some experimental results to demonstrate the performance of CS-SVM with chance constraints.
Section 5 offers all our conclusions of the research.

\section{Chance constrained CS-SVM}

In this section, we consider a conic-segmentation support vector machine with chance constraints for multiclass classification problem.
\subsection{Multiclass classification problem}
The multiclass ($n$-class) classification problem can be defined as follows:
given a data set $\Theta = \{(\bm{x}_i, y_i): i \in \mathbb{Z}_N\}$ (i.e. the training set),
where $\bm{x}_i \in \mathbb{X} \subseteq \mathbb{R}^d $ is the $i$th input vector and $y_i \in \mathbb{Z}_n$ is its corresponding class, find a decision function $h : \mathbb{X} \rightarrow \mathbb{Z}_n$ (i.e. the trained classifier) that captures the pairwise relationships of the $N$ training pairs.
Here, $\mathbb{X}$ is the input space, $\mathbb{Z}_N=\{1,\cdots,N\}$ and $\mathbb{Z}_n=\{1,\cdots,n\}$.
Meanwhile, as all the data points $\bm{x}_i$ with $ y_i =s, s\in \mathbb{Z}_n $, belong to the same class, we can denote the sample set $\Theta_s = \{\bm{x}_i: y_i=s, i \in \mathbb{Z}_N \}$ with sample size $N_s$ for $s \in \mathbb{Z}_n$.

%CS-SVM is first proposed by Shilton et al. \shortcite{Shilton2012}.
Inspired by Shilton et al.\cite{Shilton2012}, we consider the following classifiers:
\begin{equation}\label{eq_classifier}
  h(\bm{x}) =  \sigma\left(\bm{g}(\bm{x})\right),
\bm{g}(\bm{x})=\sum_{m=1}^{d}x_m\bm{w}_m + \bm{b}=W\bm{x}+\bm{b},
\end{equation}
where $\bm{x} = (x_1,\cdots,x_d)^\top$, $W = \left( \bm{w}_1, \cdots, \bm{w}_d \right) \in \mathbb{R}^{d_T \times d}$,
$\bm{g}:\mathbb{R}^{d} \rightarrow \mathbb{R}^{d_T}$ is a training machine,
$\sigma:\mathbb{R}^{d_T} \rightarrow \mathbb{Z}_n$ is a classing function,
$ \mathbb{R}^{d_T} $ is the target space,
the weight vectors $\bm{w}_m \in \mathbb{R}^{d_T}, m=1,\cdots,d$ and bias vector $\bm{b} \in \mathbb{R}^{d_T}$ are selected during training.

To retain the useful properties of the SVM formulation, the CS-SVM defines the classing function $\sigma$ based on generalized inequalities.
Each class $s \in \mathbb{Z}_n$ is associated with a proper conic class region $H_s \subset \mathbb{R}^{d_T}$,% with center $\bm{u}_{n,s}$,
such that the set of all class regions forms an almost-everywhere non-intersecting target-space covering.
%The classing function labels points in target space on the basis of which class region they fall into.
Then, the classing function $\sigma$ is defined as
\[
\sigma(\bm{a}) = s, \text{if~} \bm{a} \in \mathrm{int}(H_s). %\text{if~} \bm{a} - \bm{u}_{n,s} \in \mathrm{int}(H_s).
\]
As shown by Shilton et al. \cite{Shilton2012}, for the separable case, the trained classifier $\bm{g}$ can be found by solving the following problem:
\begin{equation}
\label{CS_original}
\begin{aligned}
\min\limits_{\bm{w}_m,\bm{b}}& ~\frac{1}{2}\sum_{m=1}^{d}\|\bm{w}_m\|^2 \\
\mathrm{s.t.}~&~ \bm{g}(\bm{x}_i)-\bm{u}_{n,s} \in H_s, \bm{x}_i \in \Theta_s, i \in \mathbb{Z}_{N_s}, s \in \mathbb{Z}_n,
\end{aligned}
\end{equation}
where the class centres $\bm{u}_{n,s}, s \in \mathbb{Z}_n$, are defined a-priori such that $\bm{u}_{n,s} \in \mathrm{int}\left(H_s\cap H_s^*\right)$ for all $s \in \mathbb{Z}_n$, $\Theta_s$.
%, and $\bm{\xi}_s$ is a random feature vector with a determined distribution.
Here $H_s^*$ is the dual cone of $H_s$ defined by
\[
H_s^* = \{ \bm{c} \in \mathbb{R}^{d_T}: \bm{c}^{\top} \bm{a} \geq 0, \forall \bm{a} \in H_s \}.
\]

This approach provides a very flexible framework for multiclass classification.
Following Shilton et al. \cite{Shilton2012}, for $ s\in \mathbb{Z}_n$, the class region $H_s$ can be defined in a recursive division form.
In this scheme, the class centers $\bm{u}_{n,s}, s\in \mathbb{Z}_n$ are defined as vertices of a regular $(n-1)$-simplex in $d_T=n-1$ dimensional target space.
The class regions are defined by the class centers using the max-projection principle:
\[
H_s = \left\{ \bm{a}: \bm{u}_{n,s}^\top\bm{a} \geq \bm{u}_{n,t}^\top\bm{a}, \forall t \ne s  \right\},
\]
where the class centres are defined recursively by
\[
\begin{array}{l}
  \bm{u}_{n,1}=\left[
\begin{array}{c}
  \bm{0} \\
  -1
\end{array}
\right] \in \mathbb{R}^{d_T},\\
\bm{u}_{n,s+1}=\frac{1}{n-1}\left[
\begin{array}{c}
  \sqrt{n(n-2)}\bm{u}_{n-1,s} \\
  1
\end{array}
\right]\in \mathbb{R}^{d_T}, s \in \mathbb{Z}_n,
\end{array}
\]
with $\bm{u}_{2,1}=[-1], \bm{u}_{2,2}=[1]$, $d_T=n-1$.

Given the training set $\Theta$, the CS-SVM can be expressed as
\begin{equation}
\label{CC_SVM_CS}
\begin{aligned}
\min\limits_{\bm{w}_m,\bm{b}}& ~\frac{1}{2}\sum_{m=1}^{d}\|\bm{w}_m\|^2 \\
\mathrm{s.t.}~&~ (\bm{v}^n_{s,t})^\top(\bm{g}(\bm{x}_i)-\bm{u}_{n,s}) \geq 0, \\
& ~~~~~~~~~~~~~~~~~~ t\ne s, t \in \mathbb{Z}_n, \bm{x}_i \in \Theta_s, i \in \mathbb{Z}_{N_s}, s \in \mathbb{Z}_n.
\end{aligned}
\end{equation}
From \cite{Shilton2012}, the CS-SVM with soft margin can be formulated as
\begin{equation}
\label{CC_SVM_CS-Soft}
\begin{aligned}
\min\limits_{\bm{w}_m,\bm{b}}& ~\frac{1}{2}\sum_{m=1}^{d}\|\bm{w}_m\|^2 + \frac{C}{N}\sum_{\stackrel{i \in \mathbb{Z}_{N_s}}{s \in \mathbb{Z}_n}}\bm{u}_{n,s}^\top\bm{\zeta}_{i,s} \\
\mathrm{s.t.}~& ~(\bm{v}^n_{s,t})^\top(\bm{g}(\bm{x}_i)+ \bm{\zeta}_{i,s}-\bm{u}_{n,s}) \geq 0,\bm{x}_i \in \Theta_s,\\
& ~ (\bm{v}^n_{s,t})^\top\bm{\zeta}_{i,s}\geq 0,t\ne s, t \in \mathbb{Z}_n, i \in \mathbb{Z}_{N_s},\\
&~ s \in \mathbb{Z}_n.
\end{aligned}
\end{equation}
Here, the second term in the objective function provides a measure of the training error on the training set, as enabled by the inclusion of the slack variables $\bm{\zeta}_{i,s}, i \in \mathbb{Z}_{N_s}, s \in \mathbb{Z}_n$.
The constant $C \in \mathbb{R}_+$ controls the trade-off between margin maximization and training error minimization.
Readers are referred to \cite{Shilton2012} for more details.

\subsection{Chance constrained formulation}

When uncertain or mislabelled data exists in the data set, the model needs to be modified to contain the uncertainty information and deal with the uncertain situation.
%Meanwhile, as all the data points $\bm{x}_i$ with $ y_i =s, s\in \mathbb{Z}_n $, belong to the same class, we can denote the sample set $\Theta_s = \{\bm{x}_i: y_i=s, i \in \mathbb{Z}_N \}$ for $s \in \mathbb{Z}_n$.
%the points in this class by $\bm{\xi}_s$ and
%And,
For each $s \in \mathbb{Z}_n $, the data set $\Theta_s$ can be viewed as the samples set generated following the distribution of random feature vector $\bm{\xi}_s$.
In other words, for $ s \in \mathbb{Z}_n$, $\bm{\xi}_s$ is a random vector defined on some probability space $(\Omega,\mathscr{F},\mathbb{P})$ with support set $\Xi\subseteq \mathbb{R}^d$.
Therefore, the following chance constraint can be introduced to ensure the small probability of misclassification for the uncertain data:
\begin{equation}\label{cc_class}
  \mathbb{P}_{F_s}\{\bm{g}(\bm{\xi}_s)-\bm{u}_{n,s} \in \mathrm{int}(H_s)\}  \geq 1-\alpha_s, s \in \mathbb{Z}_n,
\end{equation}
where, for $s \in \mathbb{Z}_n$, $F_s$ is the distribution of $\bm{\xi}_s$, $\alpha_s\in (0,1)$ is a confidence parameter and close to zero.

According to the definition of the subset $H_s, s \in \mathbb{Z}_n$, we can naturally assume that for any $ \bm{g} $, the set $\{ \bm{\xi}_s: \bm{g}(\bm{\xi}_s)) \in \mathrm{bd}(H_s)\}$ is $\mathbb{P}$-measure zero, i.e., $\bm{g}(\bm{\xi}_s)) \notin \mathrm{bd}(H_s)$ w.p.1.
Therefore, the chance constraint \eqref{cc_class} can be reformulated as
\[
\mathbb{P}_{F_s}\{\bm{g}(\bm{\xi}_s)-\bm{u}_{n,s} \in H_s\} \geq 1-\alpha_s, s \in \mathbb{Z}_n.
\]

Then, the chance constrained CS-SVM can be modeled as
\begin{equation}
\label{CC_SVM}
\begin{aligned}
\min\limits_{\bm{w}_m,\bm{b}}& ~\frac{1}{2}\sum_{m=1}^{d}\|\bm{w}_m\|^2 \\
\mathrm{s.t.}~&~ \mathbb{P}_{F_s}\{\bm{g}(\bm{\xi}_s)-\bm{u}_{n,s} \in H_s\} \geq 1-\alpha_s, s \in \mathbb{Z}_n.
\end{aligned}
\end{equation}
This model ensures an upper bound on the misclassification probability.

In some practical problems, the confidence parameters $\alpha_s \in (0,1), s\in \mathbb{Z}_n$, in the chance constrained model can be not easy to determine in advance.
In this case, the confidence parameters $\alpha_s, s\in \mathbb{Z}_n$, can be regarded as decision variables by adding a penalty term about $\alpha_s, s\in \mathbb{Z}_n$, in the objective function.

As a special case, for the binary classification, we have $n=2$ and $d_T=1$.
Then, the class region can be expressed as
\[
H_1 = \left\{ a: a \leq 0 \right\}, H_2 = \left\{ a: a \geq 0 \right\},
\]
with centers $\bm{u}_{2,1}=[-1], \bm{u}_{2,2}=[1]$.
Hence, the problem \eqref{CC_SVM} can be rewritten as
\begin{equation}
\label{CC_SVM_Binary}
\begin{array}{cl}
\min\limits_{\bm{w},b}& \frac{1}{2}\|\bm{w}\|^2 \\
\mathrm{s.t.}& \mathbb{P}_{F_1}\{\bm{\xi}_1^\top\bm{w} + b \leq -1\} \geq 1-\alpha_1, \\
&\mathbb{P}_{F_2}\{\bm{\xi}_2^\top\bm{w} + b \geq 1\} \geq 1-\alpha_2,
\end{array}
\end{equation}
where $\bm{w}=(w_1,\cdots,w_d)^\top$.

\section{Deterministic formulation of chance constrained CS-SVM}

In this section, we consider the multiclass classification based on CS-SVM, i.e., $d_T = n-1 \geq 2$.
By applying the structure of class region $H_s, s\in \mathbb{Z}_n$,
the problem \eqref{CC_SVM} can be rewritten as
\begin{equation}
\label{CC_SVM_Mul}
\begin{aligned}
\min\limits_{\bm{w}_m,\bm{b}}& ~\frac{1}{2}\sum_{m=1}^{d}\|\bm{w}_m\|^2 \\
\mathrm{s.t.}~& ~\mathbb{P}_{F_s}\{
(\bm{v}^n_{s,t})^\top(\bm{g}(\bm{\xi}_s)-\bm{u}_{n,s}) \geq 0, \\
& ~~~~~~~~~~~~~~~~~~ t\ne s, t \in \mathbb{Z}_n
\} \geq 1-\alpha_s, s \in \mathbb{Z}_n,
\end{aligned}
\end{equation}
where $\bm{v}^n_{s,t} = \bm{u}_{n,s}-\bm{u}_{n,t}$, $\bm{g}(\bm{\xi}_s) = \sum_{m=1}^{d}\xi_{s,m}\bm{w}_m + \bm{b}$.

Given the training set $\Theta = \{(\bm{x}_i, y_i): i \in \mathbb{Z}_N\}$, the data set $\Theta_s = \{\bm{x}_i: y_i=s, i \in \mathbb{Z}_N \}, s \in \mathbb{Z}_n,$ is a sample set of $\bm{\xi}_s, s \in \mathbb{Z}_n$.
Therefore, by applying sample average approximation (SAA), the model \eqref{CC_SVM_Mul} can be solved approximated by the following model:
\begin{equation}
\label{CC_SVM_SAA}
\begin{aligned}
\min\limits_{\bm{w}_m,\bm{b}}& ~\frac{1}{2}\sum_{m=1}^{d}\|\bm{w}_m\|^2 \\
\mathrm{s.t.}~& ~\frac{1}{N_s}\sum_{\stackrel{\bm{x}_i \in \Theta_s,}{i \in \mathbb{Z}_{N_s}}}\mathbf{1}_{C_s}(\bm{w}_m,\bm{x}_i) \geq 1-\alpha_s, s \in \mathbb{Z}_n,\\
\end{aligned}
\end{equation}
where
$C_s=\{(\bm{w}_m,\bm{x}_i) :(\bm{v}^n_{s,t})^\top(\bm{g}(\bm{x}_i)-\bm{u}_{n,s}) \geq 0, t\ne s, t \in \mathbb{Z}_n \}, s \in \mathbb{Z}_n$
and $ \mathbf{1}_{C_s}(\bm{w}_m,\bm{x}_i) = 1 $, if $(W,\bm{\xi}_s) \in C_s$, otherwise, $ \mathbf{1}_{C_s}(\bm{w}_m,\bm{x}_i) = 0 $.

Denote $\theta^*$ and $S$ as the optimal value and the optimal solution set of problem \eqref{CC_SVM_Mul}, respectively and $\theta_N$ and $S_N$ as the optimal value and the optimal solution set of problem \eqref{CC_SVM_SAA}, respectively.
Then, from the Theorem 1 in \cite{Chen2020}, we have $\theta_N \rightarrow \theta^*$ and $\mathbb{D}(S_N,S) \rightarrow 0$ w.p.1 as $N \rightarrow \infty$.
Here, $\mathbb{D}(A,B)=\sup\limits_{x \in A}\inf\limits_{y \in B}\|x-y\|$, for two sets $A, B \in \mathbb{R}^{d_T\cdot d}$.

Moreover, problem \eqref{CC_SVM_SAA} can be reformulated as a mixed integer programming problem
\begin{equation}
\label{CC_SVM_Integer}
\begin{aligned}
\min\limits_{\bm{w}_m,\bm{b}}& ~\frac{1}{2}\sum_{m=1}^{d}\|\bm{w}_m\|^2 \\
\mathrm{s.t.}~&~ (\bm{v}^n_{s,t})^\top(\bm{g}(\bm{x}_i)-\bm{u}_{n,s}) + M_s z^s_i \geq 0,\\
& ~~~~~~~~~~~~~~~~~~ t\ne s, t \in \mathbb{Z}_n, \bm{x}_i \in \Theta_s, i \in \mathbb{Z}_{N_s}, s \in \mathbb{Z}_n,\\
& ~\sum_{ i \in \mathbb{Z}_{N_s}} z^s_i \leq \alpha_sN_s, s \in \mathbb{Z}_n,\\
&~ z^s_i \in \{0,1\}, i \in \mathbb{Z}_{N_s}, s \in \mathbb{Z}_n,
\end{aligned}
\end{equation}
where $M_s$ is a sufficiently large positive number, such that $(\bm{v}^n_{s,t})^\top(\bm{g}(\bm{x}_i)-\bm{u}_{n,s}) + M_s \geq 0$ for all $\bm{w}_m, \bm{b}, \bm{x}_i \in \Theta_s, i \in \mathbb{Z}_{N_s}, s \in \mathbb{Z}_n$.

Comparing the models \eqref{CC_SVM_CS}, \eqref{CC_SVM_CS-Soft} and model \eqref{CC_SVM_Integer}, we can observe that under strict separation
conditions, model \eqref{CC_SVM_CS} and model \eqref{CC_SVM_CS-Soft} are equivalent.
And the solution of model \eqref{CC_SVM_CS} is also feasible for model \eqref{CC_SVM_Integer}.
In addition, by setting $\alpha_s = 0, s \in \mathbb{Z}_n$, model \eqref{CC_SVM_CS-Soft} reduces to model \eqref{CC_SVM_Integer}.
However, when there are uncertainty among the samples, model \eqref{CC_SVM_Integer} won't be affected as much as model \eqref{CC_SVM_CS} and model \eqref{CC_SVM_CS-Soft}, since only partial samples play a role during training in model \eqref{CC_SVM_Integer}.

Under non-strict separation conditions, model \eqref{CC_SVM_CS} becomes infeasible.
An optimal solution, denoted by $ (\bm\bar{{w}}_m,\bm{\bar{b}})$, can be obtained by solving model \eqref{CC_SVM_Integer}, while some points, which lead the non-separation, are ignored according to the proportion $\alpha_s$ of total samples.
Model \eqref{CC_SVM_CS-Soft} can provide a solution with slack variables $\bm{\zeta}_{i,s}, i \in \mathbb{Z}_{N_s}, s \in \mathbb{Z}_n$.
By observing the constraints in model \eqref{CC_SVM_Integer} and model \eqref{CC_SVM_CS-Soft},
it is not hard to see that there always exist some $ \bm{\bar{\zeta}}_{i,s}, i \in \mathbb{Z}_{N_s}, s \in \mathbb{Z}_n $, such that $ \left( \bm\bar{{w}}_m,\bm{\bar{b}},\bm{\bar{\zeta}}_{i,s}  \right) $ is feasible for model \eqref{CC_SVM_CS-Soft}.
However, model \eqref{CC_SVM_CS-Soft} focus more on the points, which cause the non-separation, compared with model \eqref{CC_SVM_Integer}, which just ignores these points.
Therefore the classifier provided by model \eqref{CC_SVM_CS-Soft} performs worse when these points are mislabelled.
While, the model \eqref{CC_SVM_Integer} is solved by ignoring these mislabelled data.
Therefore, the classifier obtained by solving model \eqref{CC_SVM_Integer} can perform more robustly against the mislabelled data.

\subsection{Kernelization}

As the training machine $\bm{g}$ is a linear vector-valued function on $\bm{x}$, the classifier in \eqref{eq_classifier} is modelled for linearly separable case.
For nonlinear classification, we can apply a feature map such that the nonlinearly separable case can be dealt with linear classifier.

Denote $\bm{\phi}: \mathbb{X} \rightarrow \mathbb{R}^{d_K}$ the feature map and $ \mathbb{R}^{d_K} $ the feature space.
The points in input space $\mathbb{X}$ are mapped to feature space $\mathbb{R}^{d_K}$ by the pre-defined nonlinear feature map $\bm{\phi}$, and $\phi_m:\mathbb{X} \rightarrow \mathbb{R}$ is the $m$-th component of map $\bm{\phi}$.
And in the feature space $\mathbb{R}^{d_K}$, the linear classifier can be applied.
Therefore, we can denote $\bm{\zeta}=\bm{\phi}(\bm{x})$ and re-define the training machine as
$\bm{g}:\mathbb{R}^{d_K} \rightarrow \mathbb{R}^{d_T}$ and
\[
\bm{g}(\bm{\zeta})=\sum_{m=1}^{d_K}\zeta_m\bm{w}_m + \bm{b} = W\bm{\zeta}+\bm{b},
\]
where $W=\left( \bm{w}_1, \cdots, \bm{w}_{d_K} \right) \in \mathbb{R}^{d_T \times d_K}$.
Correspondingly, the classifier can be expressed as $h(\bm{x}) =  \sigma\left(\bm{g}(\bm{\phi}(\bm{x}))\right)$.

A nonlinear decision boundary in $ \mathbb{R}^d  $ can be be obtained by solving the mixed integer programming problem \eqref{CC_SVM_Integer} in the higher-dimensional feature space $\mathbb{R}^{d_K}$:
\begin{equation}
\label{CC_SVM_Integer_K}
\begin{aligned}
\min\limits_{\bm{w}_m,\bm{b}}& ~\frac{1}{2}\sum_{m=1}^{d_K}\|\bm{w}_m\|^2 \\
\mathrm{s.t.}~&~ (\bm{v}^n_{s,t})^\top(\bm{g}(\bm{\phi}(\bm{x}_i))-\bm{u}_{n,s}) + M_s z^s_i \geq 0,\\
& ~~~~~~~~~~~~~~~~~~ t\ne s, t \in \mathbb{Z}_n, \bm{x}_i \in \Theta_s, i \in \mathbb{Z}_{N_s}, s \in \mathbb{Z}_n,\\
& ~\sum_{ i \in \mathbb{Z}_{N_s}} z^s_i \leq \alpha_sN_s, s \in \mathbb{Z}_n,\\
&~ z^s_i \in \{0,1\}, i \in \mathbb{Z}_{N_s}, s \in \mathbb{Z}_n,
\end{aligned}
\end{equation}

To carry out this problem, we need to reformulate the chance constrained problem in terms of a given kernel function $ K(\bm{x}_i, \bm{x}_j)=\bm{\phi}(\bm{x}_i)^\top\bm{\phi}(\bm{x}_j)$ satisfying Mercer's condition (\cite{shawe2004kernel}).

The kernel trick will only work if problem \eqref{CC_SVM_Integer_K} can be entirely expressed in terms of inner products of the mapped data $\bm{\phi}(\bm{x})$ only.
Fortunately, this is indeed the case as shown in the following lemma.

\begin{lemma}\label{lem_span}
Let $ \Theta_s, s \in \mathbb{Z}_n$, be the training sample set in the class corresponding to $ \bm{\xi}_s, s \in \mathbb{Z}_n $, respectively.
Denote each row of $W$ by $\omega_j^\top, j=1,\cdots, d_T$.
Then, each row $\omega_j$ of the optimal $ W $ will lie in the span of the data points in $ \Theta_s, s \in \mathbb{Z}_n$.
\end{lemma}
\begin{proof}
For $j=1,\cdots,d_T$, we can write any $\omega_j$ as $\omega_j = \omega_j^p + \omega_j^o$, where $\omega_j^p$ is the projection of $\omega_j$ in the span of the samples (vector space spanned by all the sample points in $ \Theta_s, s \in \mathbb{Z}_n$),
whereas $ \omega_j^o $ is the orthogonal component to the samples.
Then, it can be easily checked that
\[
\begin{array}{l}
  \sum_{m=1}^{d_K}\|\bm{w}_m\|^2 = \sum_{j=1}^{d_T}\left(\|\omega_j^p\|^2 + \|\omega_j^o\|^2\right), \\
  (\bm{v}^n_{s,t})^\top(\bm{g}(\bm{x}_i)-\bm{u}_{n,s}) =
  \sum_{j=1}^{d_T}(\bm{v}^n_{s,t})_j\left(( \omega_j^p )^\top\bm{x}_i + \bm{b}_j - (\bm{u}_{n,s})_j \right),
\end{array}
\]
because, for $j=1,\cdots,d_T$, $ (\omega_j^p)^\top(\omega_j^o) =0 $ and $ (\omega_j^o)^\top\bm{x}_i = 0, j=1,\cdots,d_T, \forall \bm{x}_i \in \Theta_s, s \in \mathbb{Z}_n $.
Therefore, for $j=1,\cdots,d_T$, the orthogonal component $ \omega_j^o $ won't affect the constraints in problem \eqref{CC_SVM_Integer}.
Since the objective is to be minimized, we can get $\omega_j^o = 0, j=1,\cdots,d_T$, which means $ \omega_j = \omega_j^p, j=1,\cdots,d_T$.
This implies that the optimal $  \omega_j, j=1,\cdots,d_T $, will lie in the span of the data points in $ \Theta_s, s \in \mathbb{Z}_n$.
\end{proof}

As a consequence, for $j=1,\cdots, d_T$, we can write $\omega_j$ as a linear combination of the samples and then solve for the coefficients.
By doing so, one can easily check that the optimization problem \eqref{CC_SVM_Integer} can be expressed entirely in terms of inner products between samples in the sample sets $\Theta_s, s \in \mathbb{Z}_n$, only if the conditions of the lemma are fulfilled.
This will make the kernelization of our approach possible.

\begin{theorem}\label{TH_Kernel}
Let $K_{ij}=K(\bm{x}_i, \bm{x}_j)$, where $K$ is a given kernel function satisfying Mercer’s condition.
With the conditions in Lemma \ref{lem_span}, the optimal classifier in the feature space $\mathbb{R}^{d_K}$ for data set $\Theta$ can be expressed as
\begin{equation}\label{machine_K}
  \bm{g}(\bm{\phi}(\bm{x})) = f(\gamma^*) + \bm{b}^*,
\end{equation}
where $f:\mathbb{R}^{N} \rightarrow \mathbb{R}^{d_T}$ is a map, $f_j$ is the $j$-th component of map $f$ with $f_j(\gamma^*)=\sum_{s \in \mathbb{Z}_n}\sum_{\bm{x}_i \in \Theta_s}K(\bm{x},\bm{x}_i)(\gamma_i^j)^*$.
The optimal parameter $\gamma^*$ and $\bm{b}^*$ can be obtained by solving the following problem:
\begin{equation}
\label{CC_SVM_Kernel}
\begin{aligned}
\min\limits_{\gamma^j,\bm{b}}& ~\frac{1}{2}\sum_{j=1}^{d_T}(\gamma^j)^\top K \gamma^j \\
\mathrm{s.t.}~&~ \sum_{j=1}^{d_T}(\bm{v}^n_{s,t})_j\left(K_i^\top\gamma^j + \bm{b}_j - (\bm{u}_{n,s})_j \right) + M_s z^s_i \geq 0,\\
& ~~~~~~~~~~~~~~~~~~ t\ne s, t \in \mathbb{Z}_n, \bm{x}_i \in \Theta_s, i \in \mathbb{Z}_{N_s}, s \in \mathbb{Z}_n,\\
& ~\sum_{ i \in \mathbb{Z}_{N_s}} z^s_i \leq \alpha_sN_s, s \in \mathbb{Z}_n,\\
&~ z^s_i \in \{0,1\}, i \in \mathbb{Z}_{N_s}, s \in \mathbb{Z}_n,
\end{aligned}
\end{equation}
where $K$ is a matrix with elements $K_{ij} = \bm{\phi}(\bm{x}_i)^\top\bm{\phi}(\bm{x}_j), \bm{x}_i \in \Theta_s, \bm{x}_j \in \Theta_{s'}, s,s' \in \mathbb{Z}_n $,
and $ K_i $ is a vector with elements $ K_{ij} = \bm{\phi}(\bm{x}_i)^\top\bm{\phi}(\bm{x}_j), \bm{x}_j \in \Theta_{s}, s \in \mathbb{Z}_n $.
\end{theorem}
\begin{proof}
Please refer to the Appendix \ref{AppendixA}.
\end{proof}

For the binary classification, we can get the following corollary from Theorem \ref{TH_Kernel} directly.
\begin{corollary}
For the binary classification problem, i.e. $n=2$, $d_T=1 $, the problem \eqref{CC_SVM_Kernel} can be reduced to
\begin{equation}
\label{CC_SVM_KBinary}
\begin{aligned}
\min\limits_{\gamma,\bm{b}}& ~\frac{1}{2}\gamma^\top K \gamma \\
\mathrm{s.t.}~&~ K_i^\top\gamma + \bm{b} - M_1 z^1_i \leq -1, \bm{x}_i \in \Theta_1,\\
&~ K_j^\top\gamma + \bm{b} + M_2 z^2_j \geq 1, \bm{x}_j \in \Theta_2,\\
& ~~\sum_{ i \in \mathbb{Z}_{N_1}} z_i^1 \leq \alpha_1N_1, \sum_{ j \in \mathbb{Z}_{N_2}} z_j^2 \leq \alpha_2N_2,\\
&~ z_i^1,z_j^2 \in \{0,1\}, i \in \mathbb{Z}_{N_1}, j \in \mathbb{Z}_{N_2}.
\end{aligned}
\end{equation}
In addition, the optimal classifier for binary classification can be expressed as
\[
\bm{g}(\bm{\phi}(\bm{x})) = \sum_{i \in    \mathbb{Z}_{N_1}\cup \mathbb{Z}_{N_2}}K(\bm{x},\bm{x}^i)\gamma_i^* + \bm{b}^*,
\]
where $(\gamma^*, \bm{b}^*)$ is an optimal solution for problem \eqref{CC_SVM_KBinary}, and $ \gamma_i^* $ is the $i$th component of $\gamma^*$.
\end{corollary}

\subsection{Geometric interpretation}

In this subsection, we try to interpret how the sample based reformulation of chance constrained SVM works geometrically.

Without loss of generality, the binary classification can be taken as an example.
For binary classification, based on the data points, the problem \eqref{CC_SVM_Integer} can be reduced as
\begin{equation}
        \label{CC_SVM_Integer_B}
        \begin{aligned}
        \min\limits_{\bm{w},b,z_1,z_2}& ~\frac{1}{2}\|\bm{w}\|^2 \\
        \mathrm{s.t.}~~~&~ (\bm{x}^i_1)^\top\bm{w} + b - z^i_1M_1\leq -1, \forall \bm{x}^i_1 \in \Theta_1,\\
        &~(\bm{x}^j_2)^\top\bm{w} + b + z^i_2M_2 \geq 1, \forall \bm{x}^j_2 \in \Theta_2,\\
        & ~\sum_{ i \in \mathbb{Z}_{N_1}} z^i_1 \leq \alpha_1N_1, \sum_{ j \in \mathbb{Z}_{N_2}} z^j_2 \leq \alpha_2N_2,\\
        &~ z^i_1,z^j_2 \in \{0,1\}, i \in \mathbb{Z}_{N_1}, j \in \mathbb{Z}_{N_2}.
        \end{aligned}
\end{equation}
And the CS-SVM model \eqref{CC_SVM_CS-Soft} can be expressed as
\begin{equation}
        \label{CC_SVM_CS-Soft_B}
        \begin{aligned}
        \min\limits_{\bm{w},b,\eta}& ~\frac{1}{2}\|\bm{w}\|^2 + \frac{C}{N}\left(\sum_{i \in \mathbb{Z}_{N_1}}\eta_i + \sum_{j \in \mathbb{Z}_{N_2}}\eta_j\right)\\
        \mathrm{s.t.}~&~ (\bm{x}^i_1)^\top\bm{w} + b \leq -1 + \eta_i, \forall \bm{x}^i_1 \in \Theta_1,\\
        &~(\bm{x}^j_2)^\top\bm{w} + b \geq 1 - \eta_j, \forall \bm{x}^j_2 \in \Theta_2,\\
        &~\eta_i \geq 0, i \in \mathbb{Z}_{N_1},\eta_j \geq 0, j \in \mathbb{Z}_{N_2}.
        \end{aligned}
\end{equation}

From model \eqref{CC_SVM_Integer_B} and model \eqref{CC_SVM_CS-Soft_B}, we can observe that in model \eqref{CC_SVM_Integer_B}, the nonseparable points are removed during the learning process.
This lead the margin between two classes larger.
However, in model \eqref{CC_SVM_CS-Soft_B}, nonseparable points are considered with penalization in the objective function.
This means that the nonseparable points will still affect the learning process. If the nonseparability of these points comes from noise or mislabel, the classifier from model \eqref{CC_SVM_CS-Soft_B} will be influenced more than the classifier obtained by solving model \eqref{CC_SVM_Integer_B}.
This implies that the classifier provided by the chance constrained model can be more robust with uncertain data points.

The geometric interpretation of sample based reformulation is shown in Figure \ref{Fig_dataSAA}.
The black dashed line is the true classifier, which is a hyperplane.
The blue squares belong to class $H_1$, and the red triangles belong to class $H_2$.
In addition, we assume that there exists uncertainty in the data set, which leads misclassification and some outliers.
The mislabelled points and outliers are remarked with circles.

\begin{figure}[htbp]
\centering
\includegraphics[width=0.9\textwidth]{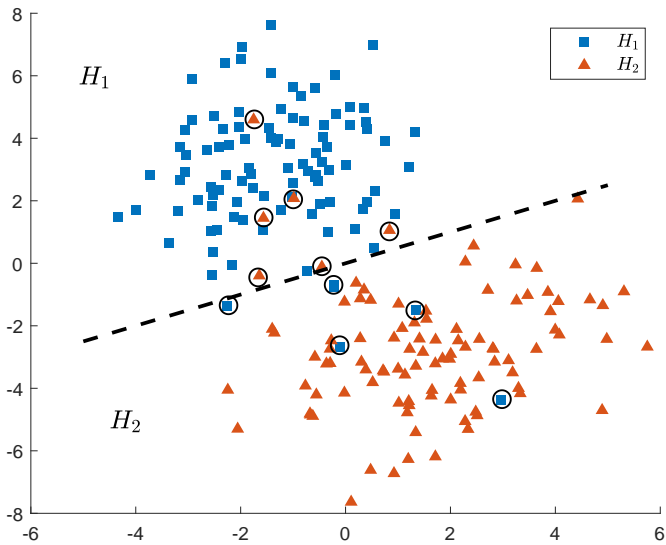}
\caption{Geometric interpretation of chance constrained CS-SVM}
\label{Fig_dataSAA}
\end{figure}

To guarantee the feasibility of the corresponding optimization problem, the mislabelled or the uncertainty points, which are far away from the boundary, will be ignored, such as the three circled blue square in the center of red triangles and the three circled red triangles in the center of blue squares in Figure \ref{Fig_dataSAA}.
Therefore, the effect of these mislabelled points will be reduced.
Since the objective is to maximize the margin of separation, some of the data points, which are close to the boundary between two regions, will not be taken into consideration either, for example, the circled blue squares and circled red triangles close to the black dashed line in Figure \ref{Fig_dataSAA}.
Ignoring these points can reduce the influence of the uncertainty or the misclassification in the data set as possible as we can.
When the mislabelled data or the uncertainty is close to the boundary, as this kind of data will be ignored during training, the mislabelled data or the uncertainty won't effect our classifier much, either.
And the total proportion of ignored data is smaller or equal to $\alpha_s, s \in \mathbb{Z}_n$.

\section{Experimental methodology and results}

In this section, the performance of chance constrained conic-segmentation support vector machine on artificial and real data is considered, especially in the situation that there are uncertainty or mislabelled points in the data set.
In the experiments, we study the performance of chance constrained support vector machine by comparing with standard conic-segmentation support vector machine and other classification model for uncertain data,
%compare with the standard conic-segmentation support vector machine,
in both binary classification and multiclass classification.

All the optimization problems in the experiments were solved by the CVX package \cite{cvx} with Matlab R2020a, on a Laptop with an Intel Core i7-8550U CPU and 16.0 GB RAM.
To solve the mixed integer programming problem, we use the command \texttt{cvx\_solver gurobi} in CVX to invoke the commercial solver Gurobi and set the time limit as 7200 seconds.

In these experiments, we have used the linear kernel, polynomial kernel and radial basis function (RBF) kernel functions:
\[
\begin{array}{l}
  K_{linear}(x,y) = x^\top y, \\
  K_{poly}(x,y) = (1+x^\top y)^d, \\
  K_{RBF} = \exp\left(-\gamma\|x-y\|^2\right),
\end{array}
\]
where the kernel parameters are $d \in \{ 2,4,6 \}$ and $10^{-2} \leq \gamma \leq 10^2$, respectively.
%To compare the performances, in each experiment, the parameters of kernel function remain the same for both standard conic-segmentation support vector machine and chance constrained model.
%The trade-off parameter $C$ in standard model was always selected as $C=1$.
%{\color{blue}
The trade-off parameter $C$ in standard conic-segmentation support vector machine was selected from $ 10^{-2} \leq \frac{C}{N} \leq 10^2 $.
Both $C$ and the kernel parameter $\gamma$ were carried out using a grid search to minimise leave-one-out error measured on the training set.
%}

%{\color{blue}
As the confidence parameters $\alpha_s \in (0,1), s\in \mathbb{Z}_n$, in the chance constrained model are not easy to predetermine properly for practical data set, in this experiment $\alpha_s, s\in \mathbb{Z}_n$, are viewed as decision variables by adding a penalty term $\sum_{s \in \mathbb{Z}_n}\alpha_s$ in the objective function.
%}

\subsection{Artificial data}

This experiment investigates the difference between the performances of the chance constrained CS-SVM and the CS-SVM with soft margin both in binary classification and multiclass classification, especially in the condition that some data points are mislabelled.

For the binary classification, a $2$-dimensional data set was designed for easy visualization.
The data set was generated in $ 4 $ different ways.
The first data set was generated from the distributions:
\[
\mathbf{P}_1 = \mathcal{N}\left(\left[
\begin{array}{c}
  1.5 \\
  -1.5
\end{array}
\right],
\left[
\begin{array}{cc}
  2 & 0.5\\
  0.5 & 3
\end{array}
\right]
\right),
\quad \mathbf{P}_2 = \mathcal{N}\left(\left[
\begin{array}{c}
  -1.5 \\
  1.5
\end{array}
\right],
\left[
\begin{array}{cc}
  2 & 0.5\\
  0.5 & 3
\end{array}
\right]
\right),
\]
where $\mathcal{N}$ is the normal distribution.
The rest three data sets were all generated randomly from the distributions:
\[
\mathbf{P}_1 = \mathcal{N}\left(\left[
\begin{array}{c}
  0 \\
  -2.5
\end{array}
\right],
\left[
\begin{array}{cc}
  2 & 0.5\\
  0.5 & 3
\end{array}
\right]
\right),
\quad \mathbf{P}_2 = \mathcal{N}\left(\left[
\begin{array}{c}
  0 \\
  2.5
\end{array}
\right],
\left[
\begin{array}{cc}
  2 & 0.5\\
  0.5 & 3
\end{array}
\right]
\right).
\]
In each of the four instances, all the data points were divided into two classes by a curve $y=f(x)$ manually.
For the first instance, the function $f(x)$ was defined as the perpendicular bisector between the centers of two distributions $\mathbf{P}_1$ and $\mathbf{P}_2$.
In the other three instances, the curve was defined by a triangular function.
Concretely, the function was expressed as $ y=\sin(\frac{\pi}{4}x) $, $y=2\sin(\frac{\pi}{2}x)$ and $y=2\sin(\pi x)$ for the second, third and fourth instances, respectively.
100 data points were generated for both classes, where 30 points were selected randomly as training data, the rest points as testing data.
In addition, the 30 training points of class 1 will be mislabelled as class 2 randomly with a probability $20\%$.

All the generated data points are shown in Figure \ref{Fig_Datasets1} - Figure \ref{Fig_Datasets4}, where the red points belong to class 1 and blue ones belong to class 2, and the black dashed curve is the theoretical curve of $ y=f(x) $, which separates the two classes.
In each figure, the left one shows all the generated data points, while the right one presents the training data with some mislabelled points.

\begin{figure}[htbp]
\centering
\begin{minipage}[t]{0.49\textwidth}
 \includegraphics[width=0.9\textwidth]{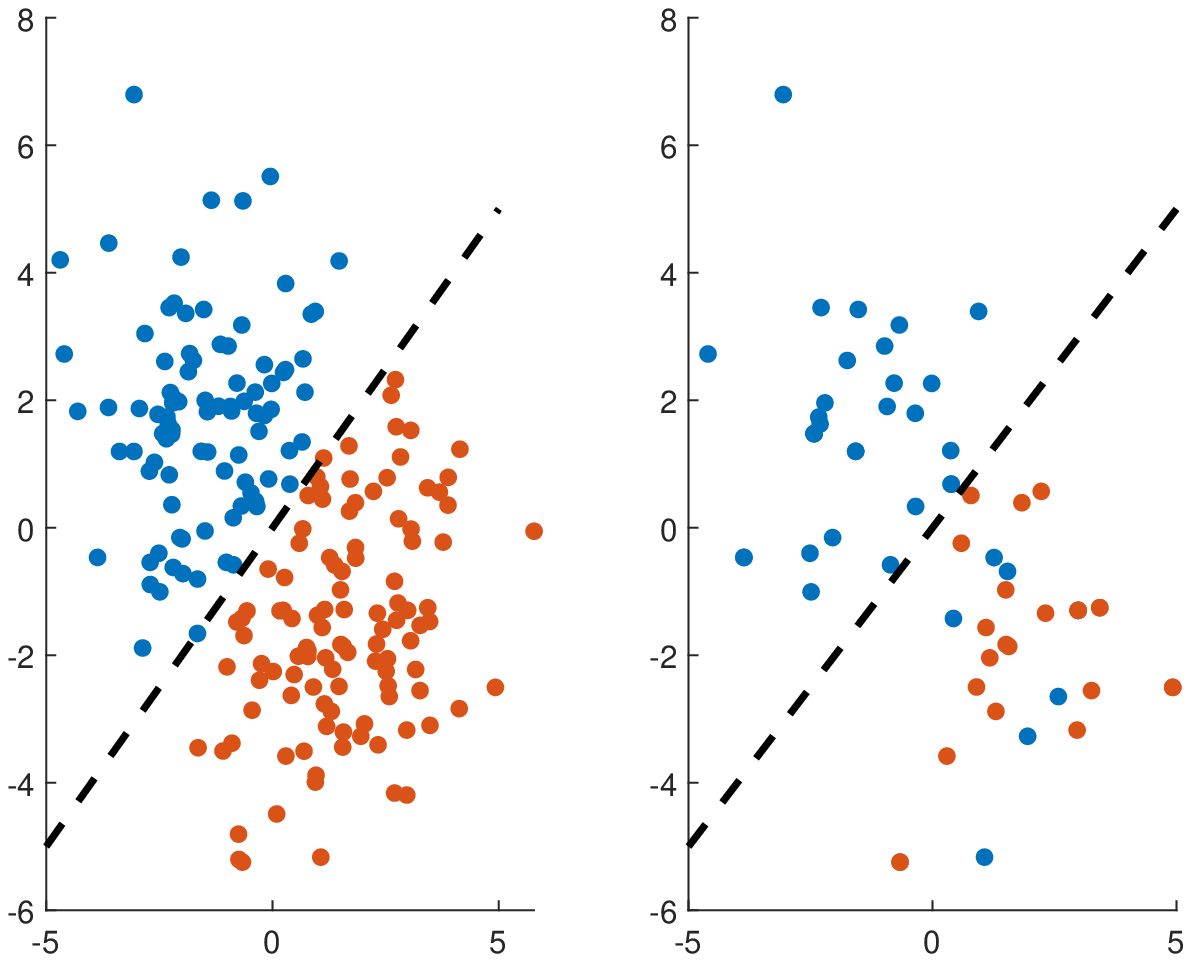}
\caption{Data set of instance 1}\label{Fig_Datasets1}
\end{minipage}
\hfill
\begin{minipage}[t]{0.49\textwidth}
 \includegraphics[width=0.9\textwidth]{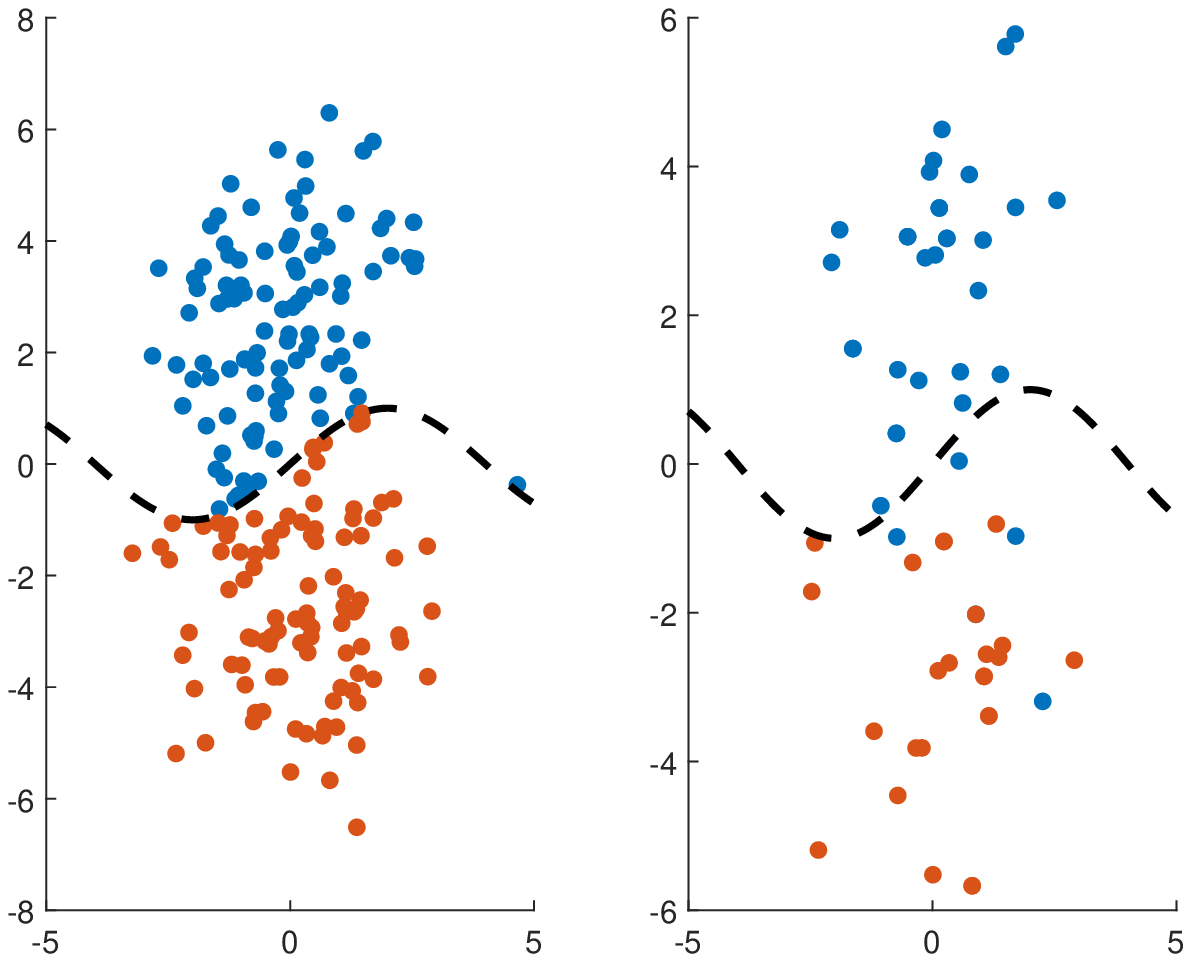}
\caption{Data set of instance 2}\label{Fig_Datasets2}
\end{minipage}
\hfill
\begin{minipage}[t]{0.49\textwidth}
 \includegraphics[width=0.9\textwidth]{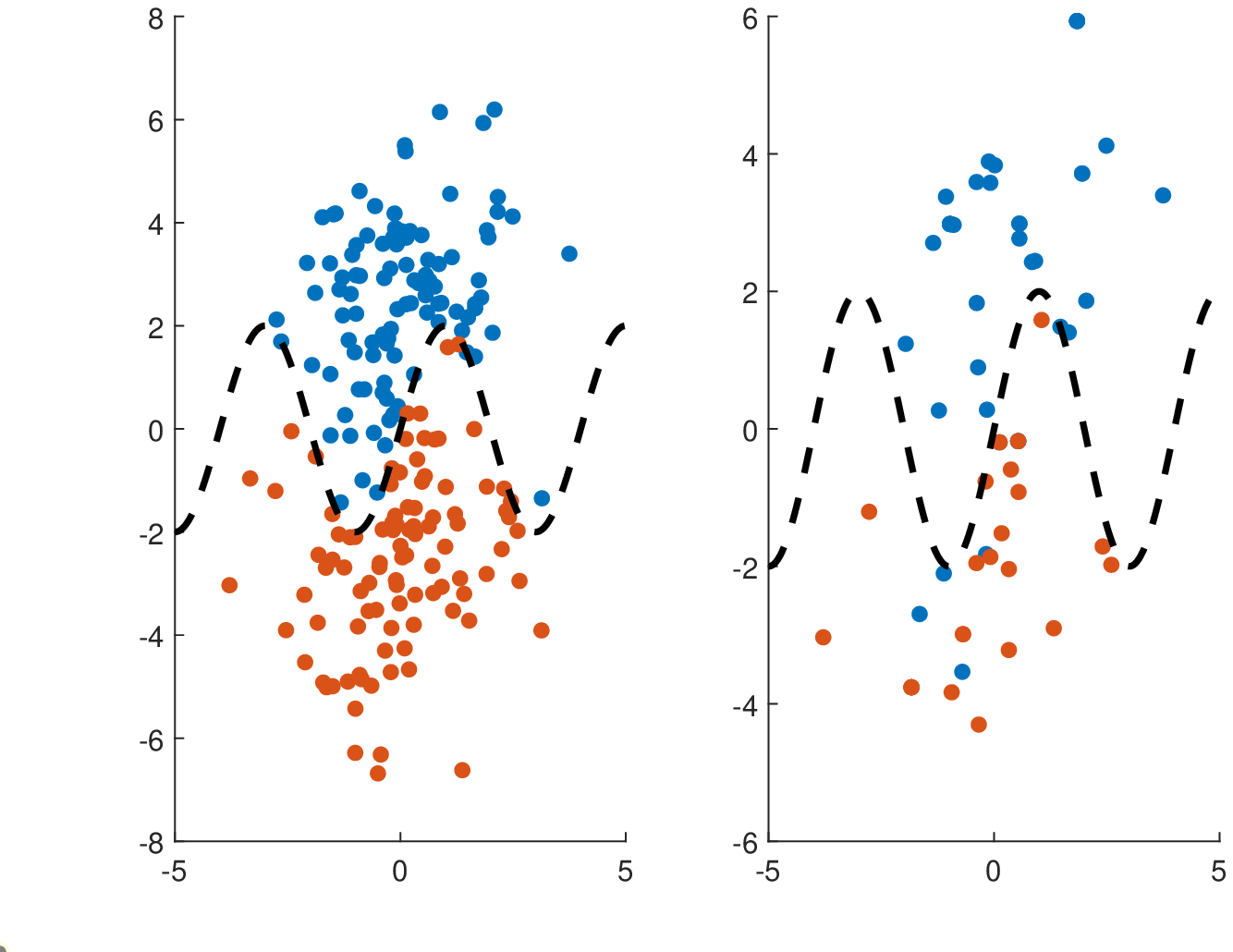}
\caption{Data set of instance 3}\label{Fig_Datasets3}
\end{minipage}
\hfill
\begin{minipage}[t]{0.49\textwidth}
 \includegraphics[width=0.9\textwidth]{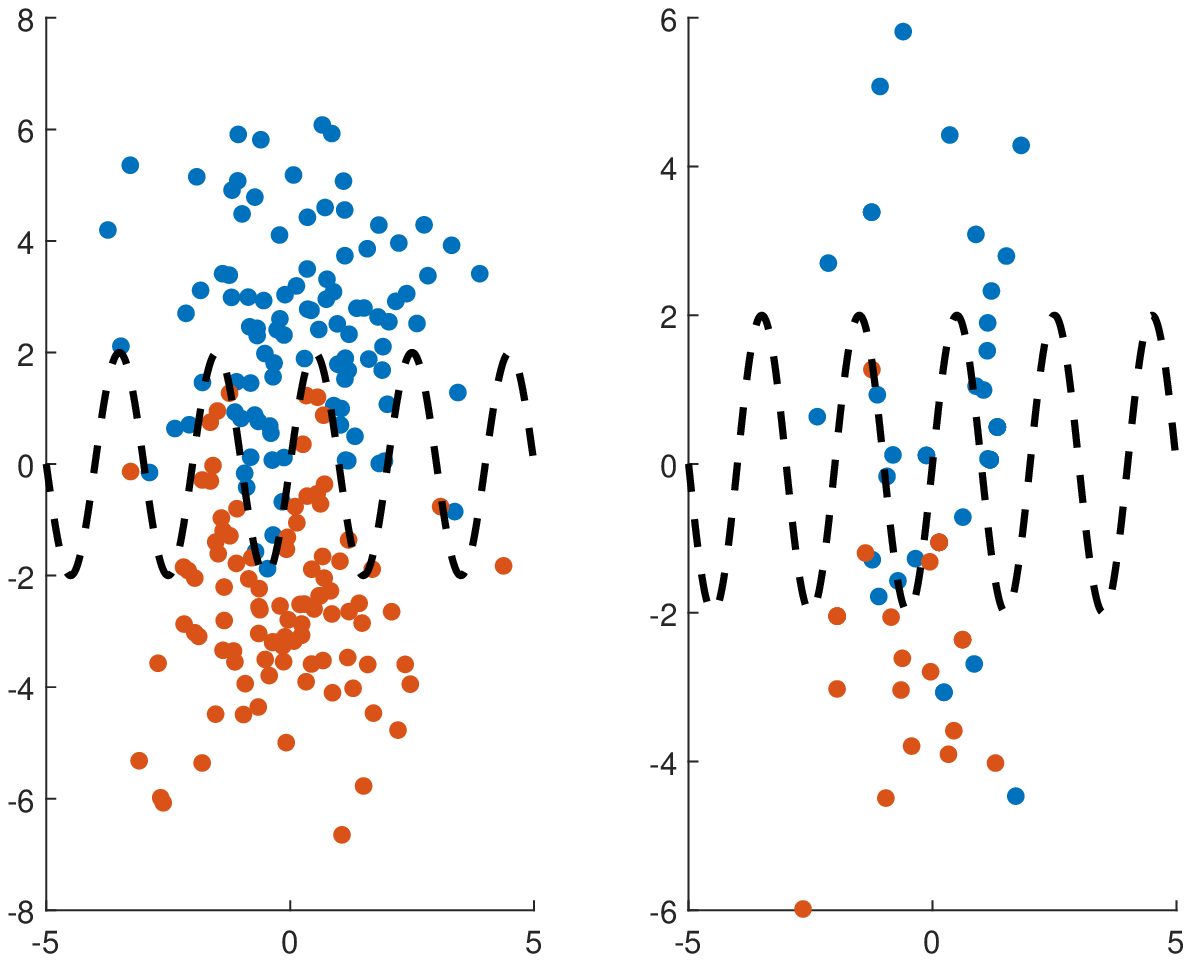}
\caption{Data set of instance 4}\label{Fig_Datasets4}
\end{minipage}
\end{figure}

The classification accuracy of both the standard support vector machine and our chance constrained model with different kernel functions are summarized in Table \ref{Table_Binary}.
The first column shows the numbers of instances, while the second the column indicates the model, where ``CS" refers to the CS-SVM with soft margin \eqref{CC_SVM_CS-Soft} by Shilton et al. \cite{Shilton2012} and ``CC" refers to the chance constrained CS-SVM \eqref{CC_SVM_Integer}.
These notations will be also used in the following tables.
The third column to seventh column present the classification accuracy with different kernel functions, i.e., linear kernel (LIN), polynomial kernel with $d=2,4,6$ and RBF kernel,
respectively.

\begin{table}[htbp]
\begin{center}
\begin{minipage}{\textwidth}
\caption{The accuracy with different kernels for binary classification}\label{Table_Binary}
%\label{Tab_first}
\begin{tabular*}{\textwidth}{@{\extracolsep{\fill}}cllllll@{\extracolsep{\fill}}}
\toprule
        \multirow{2}{*}{Instance}            &      \multirow{2}{*}{Model} & \multicolumn{5}{c}{Kernels}    \\ \cmidrule(lr){3-7}
                &  & \multicolumn{1}{c}{LIN} & \multicolumn{1}{c}{$d=2$} & \multicolumn{1}{c}{$d=4$}& \multicolumn{1}{c}{$d=6$} & \multicolumn{1}{c}{RBF} \\
\midrule
\multirow{2}{*}{1}  & CS             & 97.1\%   & 90.0\%  &  75.7\%  & 72.9\%  &  74.3\%      \\
                    & CC    & 100.0\%  & 91.4\%  &  88.6\%  & 92.9\%  &  95.7\%    \\
                    \midrule
\multirow{2}{*}{2}  & CS             & 85.7\%   & 84.3\%  & 74.3\% &  75.7\%   &  86.4\%   \\
                    & CC    & 85.0\%   & 90.0\%  & 84.3\% &  83.6\%   &  88.6\%   \\
                    \midrule
\multirow{2}{*}{3}  & CS             & 82.9\%   & 85.7\% & 71.4\%  & 64.3\% &  57.1\%   \\
                    & CC    & 82.1\%   & 87.1\% & 83.5\%  & 81.4\% &  84.3\%   \\
                    \midrule
\multirow{2}{*}{4}  & CS             & 80.0\%   & 85.7\%  &  81.4\%  & 74.3\%  &  51.4\%    \\
                    & CC    & 81.4\%   & 90.0\%  &  88.6\%  & 82.9\%  &  85.7\%    \\
\botrule
\end{tabular*}
\end{minipage}
\end{center}
\end{table}

From Table \ref{Table_Binary}, we can observe that the chance constrained CS-SVM model significantly outperforms the standard model for all the instances, especially when polynomial and RBF kernels are applied.
For all the instances, the difference between the accuracy of two models increases in general as the nonlinearity of the kernel function increase.
For example, as the degree of polynomial kernel function increases from $2$ to $6$, the accuracy of both models decreases.
However, the accuracy of CS-SVM with soft margin model decreases sharply.
The decrease of the accuracy can be larger than $20\%$.
At the same time, the decrease of accuracy for chance constrained CS-SVM won't be larger than $10\%$.
When RBF kernel function is applied, in most instances, the accuracy of CS-SVM with soft margin model is even worse.
While, the chance constrained CS-SVM with RBF kernel can generally perform very well with high accuracy.
The reason is that when the degree of polynomial kernel increases and $\gamma$ in RBF kernel is not small, the models will focus more on the individual data points.
This leads that the mislabelled points will influence more and more in the training process.
Because of the property of the chance constrained CS-SVM, the mislabelled data points will be ignored with a given probability.
This reduces the influence of mislabelled data points on the trained classifier provided by chance constrained CS-SVM.

From Table \ref{Table_Binary}, it is not hard to notice that in the first instance, which is a linear separable case, the chance constrained CS-SVM with linear kernel performs best, while the CS-SVM with soft margin also provides a high accuracy.
For the rest three instances, since they are all nonlinear separable cases, the performance of linear kernel becomes worse as the nonlinearity becomes more clear.
And in these three instances, the polynomial kernel and RBF kernel become to play an important role in the training process, which increase the accuracy compared with linear kernel.
This is consistent with the fact that the data sets can be classified easier after mapping to the feature space via the feature map $\bm{\phi}$.

\begin{figure}[htbp]
\centering
\begin{minipage}[t]{0.49\textwidth}
\includegraphics[width=1\textwidth]{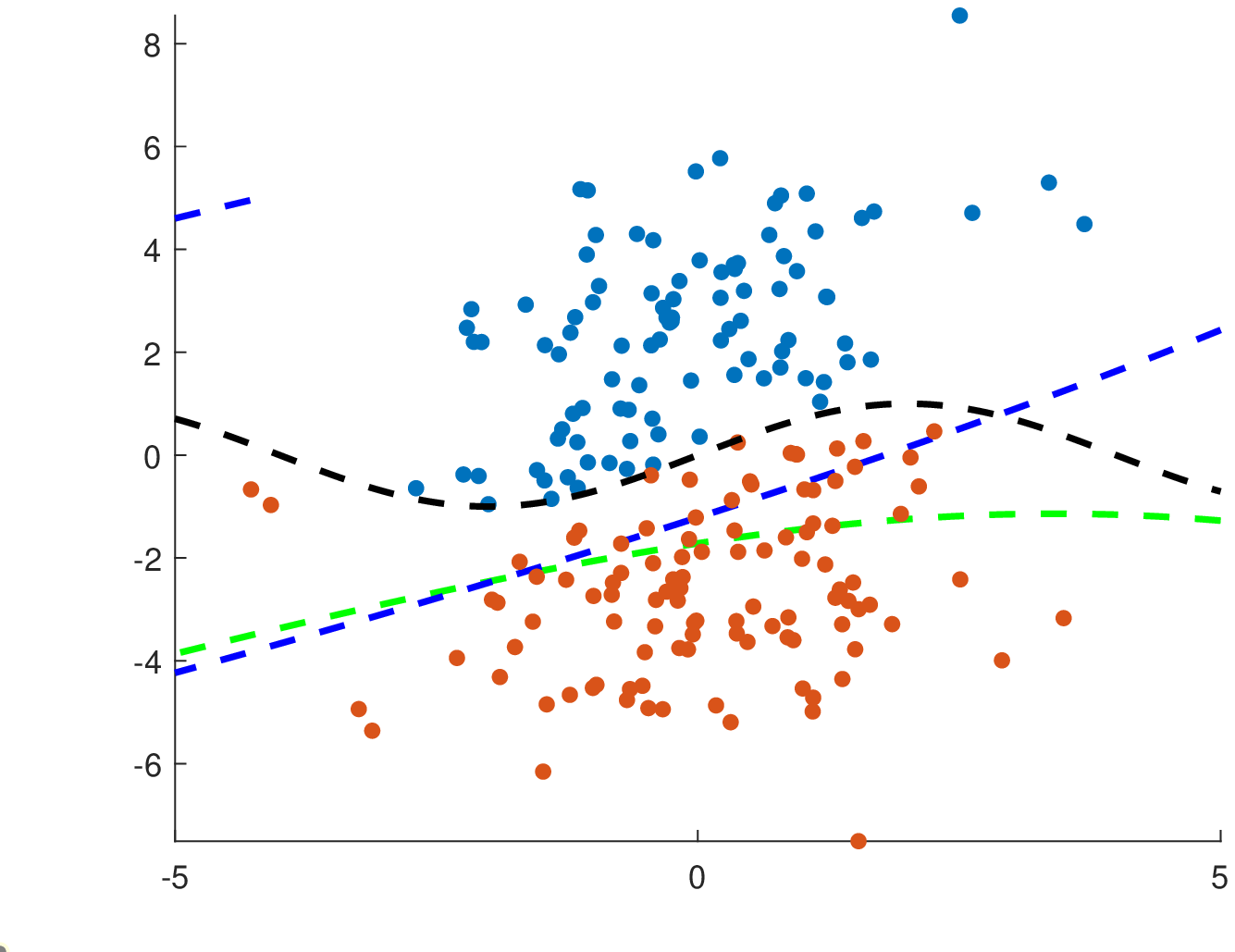}
\caption{Classification result with polynomial kernel for $d=2$}\label{Fig_KernelB2}
\end{minipage}
\hfill
\begin{minipage}[t]{0.49\textwidth}
\includegraphics[width=1\textwidth]{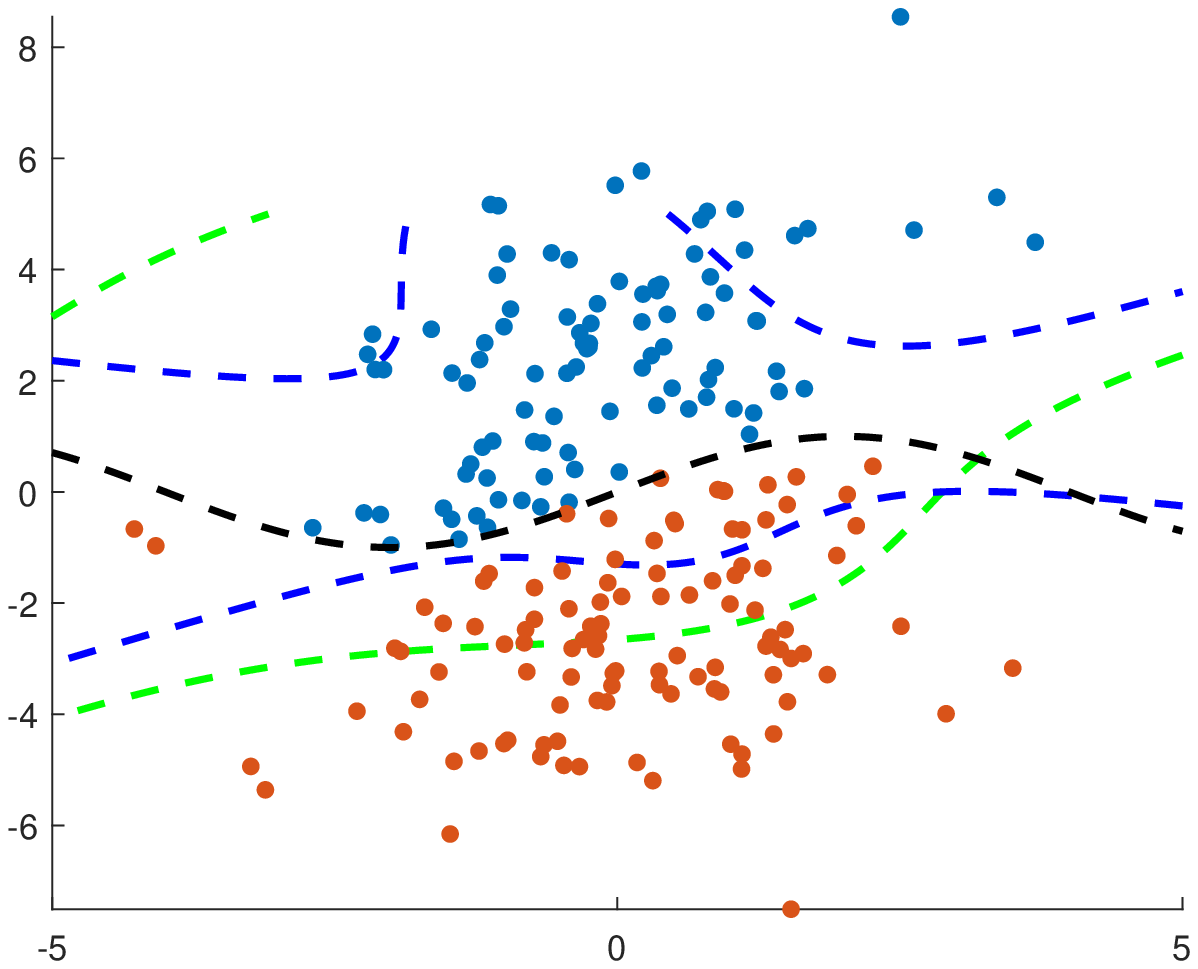}
\caption{Classification result with polynomial kernel for $d=4$}\label{Fig_KernelB3}
\end{minipage}
\end{figure}

\begin{figure}[htbp]
\centering
\begin{minipage}[t]{0.49\textwidth}
\includegraphics[width=1\textwidth]{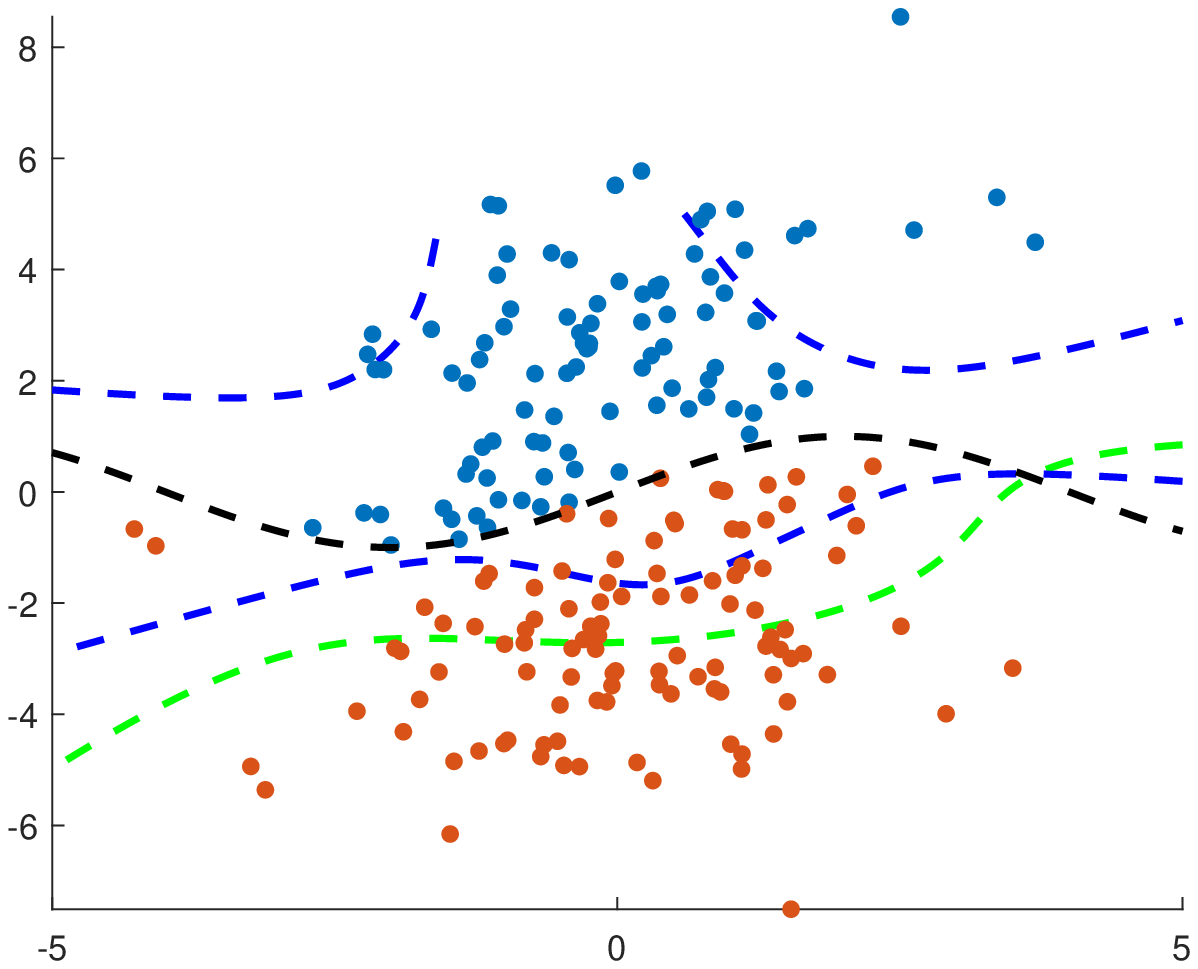}
\caption{Classification result with polynomial kernel for $d=6$}\label{Fig_KernelB4}
\end{minipage}
\hfill
\begin{minipage}[t]{0.49\textwidth}
\includegraphics[width=1\textwidth]{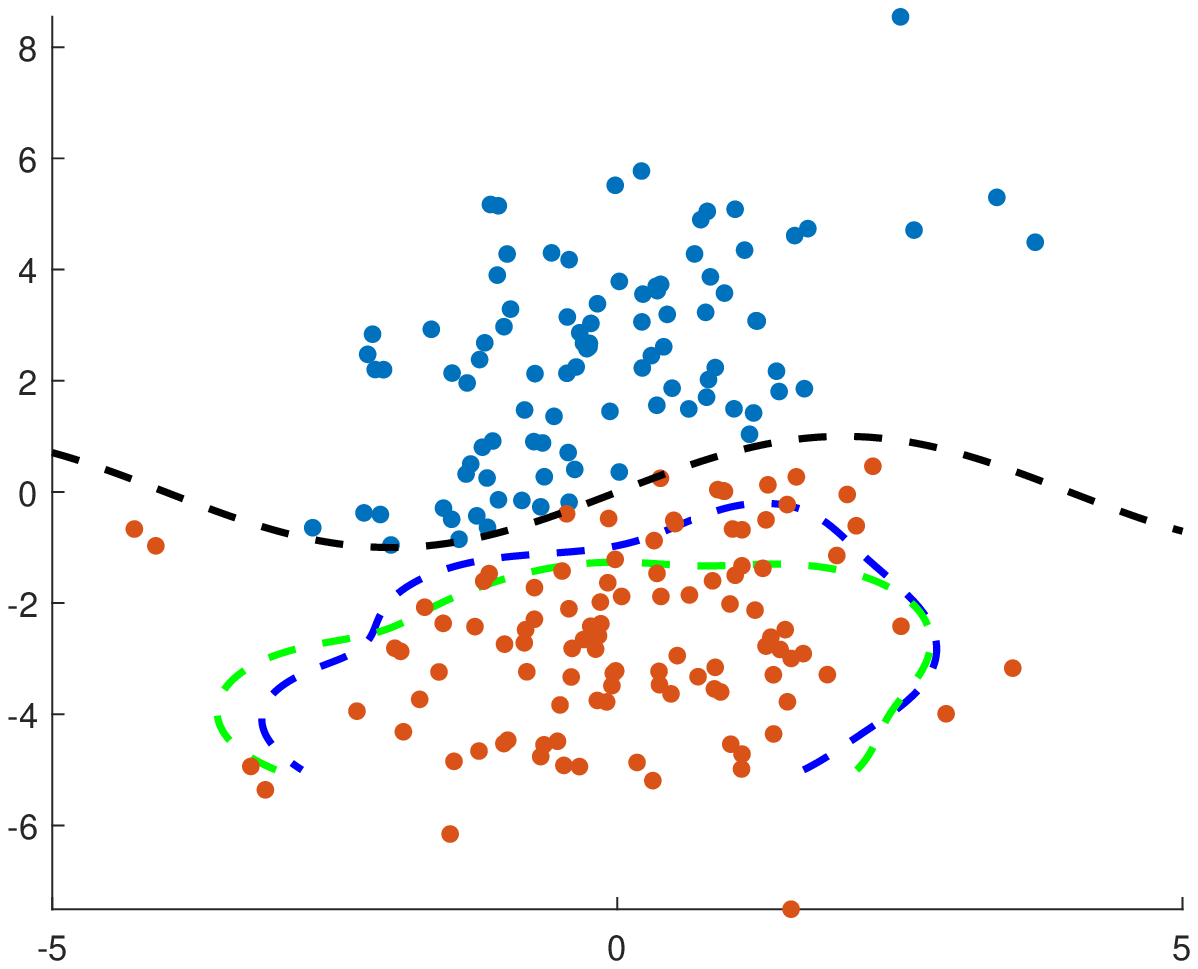}
\caption{Classification result with RBF kernel}\label{Fig_KernelB5}
\end{minipage}
\end{figure}

A representative result for the CS-SVM with soft margin and chance constrained CS-SVM is shown in Figure \ref{Fig_KernelB2} - Figure \ref{Fig_KernelB5}.
The four figures show the classification results of the second instance with different kernels.
In these four figures, the green dashed curve represents the classifier trained by the CS-SVM with soft margin, the blue dashed curve indicates the trained classifier by the chance constrained CS-SVM, and the black dashed curve is the theoretical curve which separate the two classes.
Combined with Figure \ref{Fig_Datasets3}, it is not hard to observe that the classifier is closer to the mislabelled data points, which tries to contain the mislabelled points in the other side.
This phenomenon becomes more clear as the parameter $d$ increases from $2$ to $6$.
It shows that the trained classifier was influenced a lot by these mislabelled data points.
For the trained classifier by the chance constrained CS-SVM, the influence from the mislabelled data points is not large.
With different kernels, the trained classifier by the chance constrained CS-SVM is always close to the theoretical curve, i.e., the black dashed line.
This also explained why the classification accuracy of chance constrained CS-SVM is always better.

For multiclass classification, we designed a $2$-dimensional, $3$-class data set, which was generated from the distributions:
\[
\begin{array}{l}
  \mathbf{P}_1 = \mathcal{N}\left(\left[
\begin{array}{c}
  0 \\
  0
\end{array}
\right],
\left[
\begin{array}{cc}
  1 & 0\\
  0 & 1
\end{array}
\right]
\right), \\
  \mathbf{P}_2 = \mathcal{N}\left(\left[
\begin{array}{c}
  8 \\
  -5.5
\end{array}
\right],
\left[
\begin{array}{cc}
  1.5 & 3\\
  3 & 8
\end{array}
\right]
\right), \\
\mathbf{P}_3 = \left\{\left[
\begin{array}{c}
  r\cos\theta \\
  r\sin\theta
\end{array}
\right]
:
r \in \mathcal{N}(6, 2.25), \theta \in \mathcal{U}\left(-\frac{\pi}{2}, \frac{\pi}{2}\right)
\right\},
\end{array}
\]
where $\mathcal{U}$ is the uniform distribution.
Similar with the binary classification, 100 data points were generated for each class, where 30 points were selected randomly as training data, the rest points as testing data.
And with a probability $20\%$, the 30 training points of class 3 will be mislabelled as class 1 or class 2 randomly with equal probability.

The generated data points are shown in Figure \ref{Fig_datasetMul}, where the blue points belong to class 1, red ones belong to class 2 and yellow points belong to class 3.
In the figure, the left one shows all the generated data points, while the right one presents the training data with some mislabelled points.

\begin{figure}[htbp]%
\centering
\includegraphics[width=0.9\textwidth]{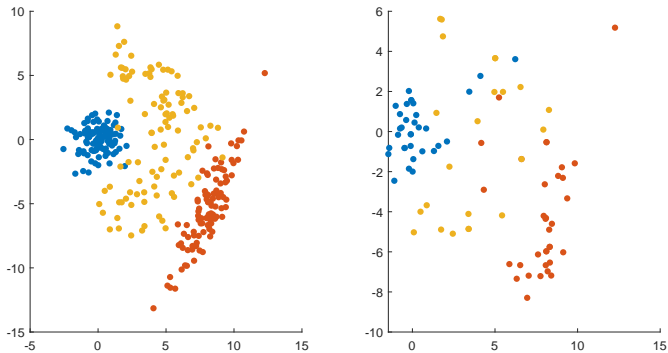}
\caption{Data set for multiclass classification}\label{Fig_datasetMul}
\end{figure}

The classification accuracy of both the CS-SVM and our chance constrained CS-SVM with different kernel functions are summarized in Table \ref{Table_Mul}.
The first column indicates the actual classification of testing data points, while the second the column indicates the model.
The third column to seventh column present the classification accuracy with different kernel functions, i.e., linear kernel (LIN), polynomial kernel with $d=2,4,6$ and RBF kernel.

\begin{table}[htbp]
\begin{center}
\begin{minipage}{\textwidth}
\caption{The accuracy with different kernels for multiclass classification}
\label{Table_Mul}
\begin{tabular*}{\textwidth}{@{\extracolsep{\fill}}cllllll@{\extracolsep{\fill}}}
\toprule
        \multirow{2}{*}{Class}            &      \multirow{2}{*}{Model} & \multicolumn{5}{c}{Kernels}    \\ \cmidrule(lr){3-7}
                &  & \multicolumn{1}{c}{LIN} & \multicolumn{1}{c}{$d=2$} & \multicolumn{1}{c}{$d=4$}& \multicolumn{1}{c}{$d=6$} & \multicolumn{1}{c}{RBF} \\
\midrule
\multirow{2}{*}{1}  & CS             & 97.1\%   & 100.0\%  & 98.6\%  & 100.0\% &  100.0\%       \\
                    & CC    & 91.4\%   & 100.0\%  & 98.6\%  & 95.7\% &   94.3\%   \\
                    \midrule
\multirow{2}{*}{2}  & CS             & 100.0\%    & 98.6\%  & 95.7\%  & 94.3\% &  90.0\%       \\
                    & CC    & 100.0\%   & 98.6\%  & 91.4\% & 90.0\% &    87.1\%       \\
                    \midrule
\multirow{2}{*}{3}  & CS             & 61.4\%    & 75.7\%  & 64.3\%  & 34.3\% &   45.7\%      \\
                    & CC    & 78.6\%   & 94.3\%  & 90.0\%  & 87.1\% &   94.3\%    \\
                    \midrule
\multirow{2}{*}{Total}  & CS         & 86.7\%    & 91.4\%  & 86.2\%  & 75.7\% &  78.6\%           \\
                    & CC    & 90.0\%   & 97.1\%  & 93.3\%  & 91.0\% &  91.9\%     \\
\botrule
\end{tabular*}
\end{minipage}
\end{center}
\end{table}

As shown in Table \ref{Table_Mul}, we can observe that the total accuracy of chance constrained CS-SVM is always larger than the one of CS-SVM with soft margin.
For both models, the performance with linear kernel is worst, which is mainly because the data set is not linearly separable, especially between class 1 and class 3.
From Figure \ref{Fig_datasetMul}, it is not hard to observe that the class 2 and classes 1,3 are linearly separable.
This is the reason why the accuracy of class 2 is $100\%$ when linear kernel is applied.
When polynomial kernel and RBF kernel are applied, the accuracy of class 3 can be improved.
In addition, the accuracy of class 2 decreases, as polynomial kernel and RBF kernel focus more on individual data points.
It is mainly due to the fact that the training data points is a small part of the whole data set.
Therefore, the kernels, especially RBF kernel, will ignore some information of the whole data set, which leads low accuracy in testing data.
We also noticed that the accuracy of class 1 and class 2 are always high, while the accuracy of class 3 is the lowest.
It is because there are some mislabelled points in the training data set of class 3.
These mislabelled points will push the classifier deep into the region of class 3,
which leads the sacrifice of accuracy of class 3 to guarantee the accuracy of class 1 and class 2.

However, the sacrifice of accuracy of class 3 is not that obvious in the performance of chance constrained CS-SVM.
For class 1 and class 2, the accuracy of both CS-SVM with soft margin and chance constrained CS-SVM is high and comparable, though the accuracy of CS-SVM with soft margin is slightly larger than the chance constrained CS-SVM.
However, for class 3, the accuracy of classification by chance constrained CS-SVM is always much larger than the accuracy of classification by CS-SVM with soft margin.
The chance constrained CS-SVM can generally provide an accuracy around $88\%$ with different polynomial kernels.
When RBF kernel is applied, the accuracy of class 3 by chance constrained CS-SVM can reach $94.0\%$.
This is consistent with the case of binary classification that the chance constrained CS-SVM can reduce the influence of the mislabelled data points on the trained classifier efficiently.

The classification regions for polynomial kernel with $d=2$ is shown in Figure \ref{Fig_KernelM1} and Figure \ref{Fig_KernelM2}, where the Figure \ref{Fig_KernelM1} shows the classification result of chance constrained CS-SVM, and the Figure \ref{Fig_KernelM2} shows the result of CS-SVM with soft margin.
In both Figure \ref{Fig_KernelM1} and Figure \ref{Fig_KernelM2}, the blue circles refer to the points in class 1, the red triangles refer to the data in class 2, and the yellow diamonds represent the data points in class 3.
The blue region means the region classified as class 1, the red region indicates the region classified as class 2 and the yellow region represents the region of class 3.

\begin{figure}[htbp]
\centering
\begin{minipage}[t]{0.49\textwidth}
\includegraphics[width=1\textwidth]{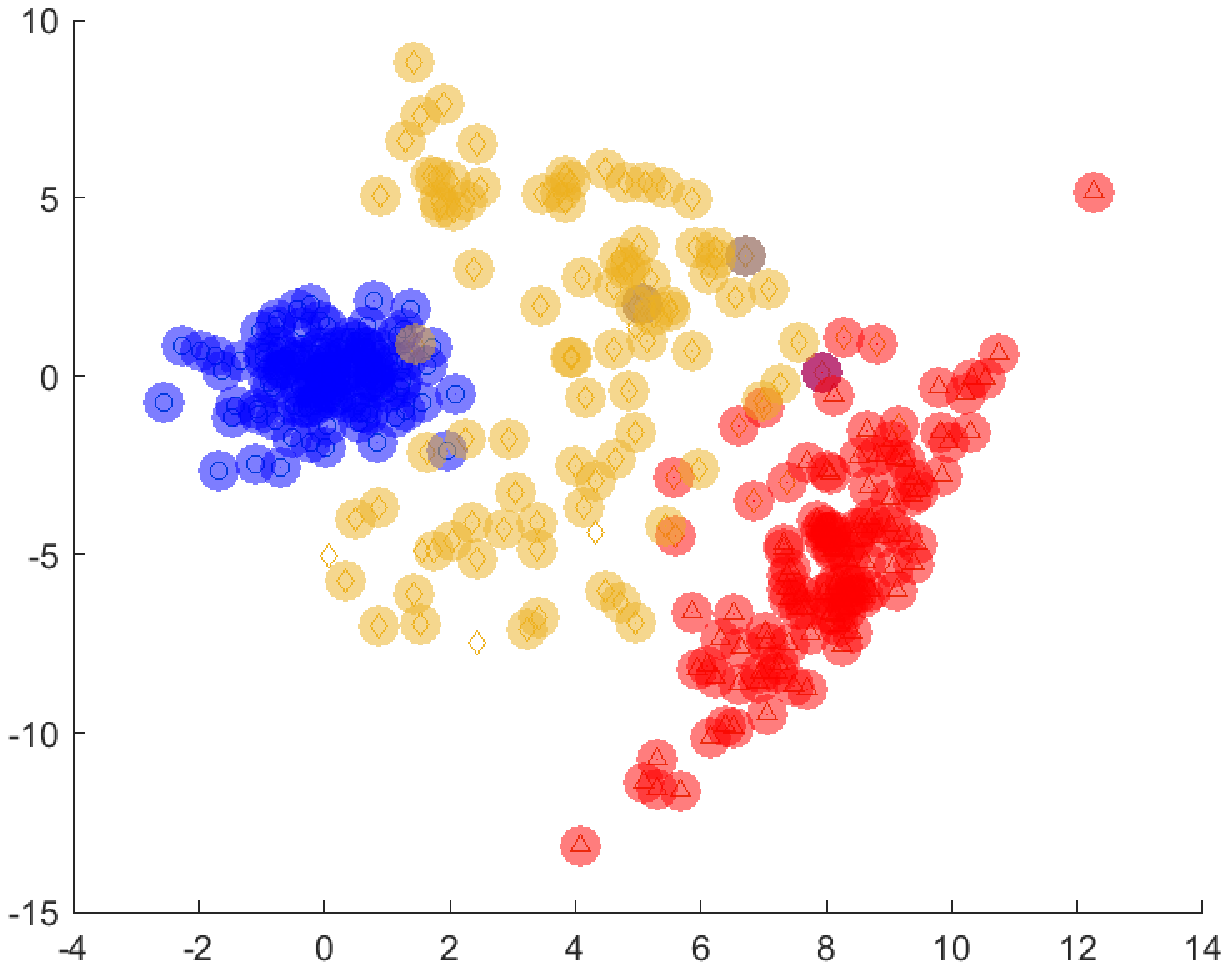}
\caption{Multiclass classification of chance constrained CS-SVM with polynomial kernel for $d=2$}\label{Fig_KernelM1}
\end{minipage}
\hfill
\begin{minipage}[t]{0.49\textwidth}
\includegraphics[width=1\textwidth]{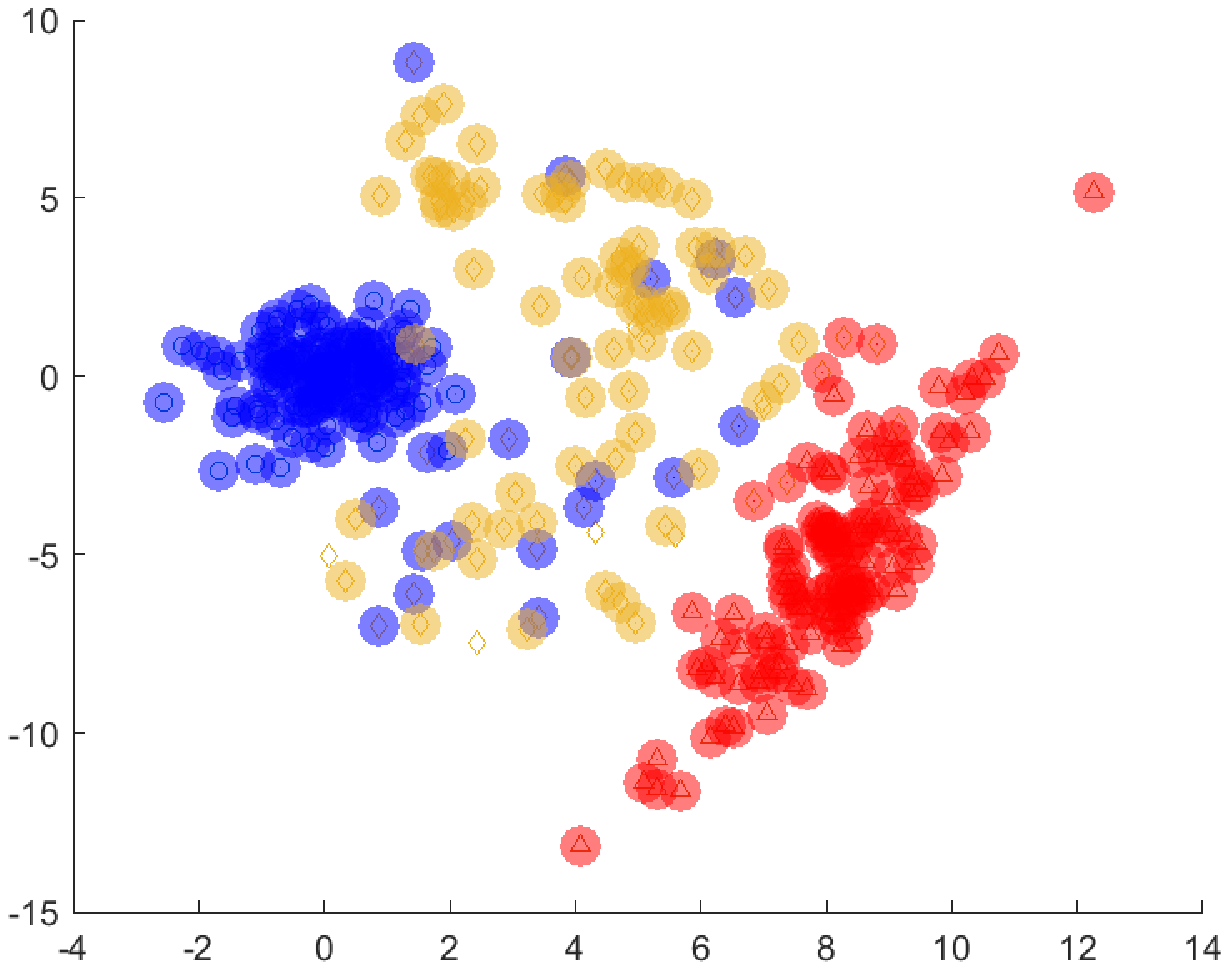}
\caption{Multiclass classification of CS-SVM with soft margin with polynomial kernel for $d=2$}\label{Fig_KernelM2}
\end{minipage}
\end{figure}

Compared these two figures, it is not hard to notice that many parts in the region of class 3 is recognized as part of region of class 1 or class 2 by the CS-SVM with soft margin.
This is mainly due the influence of the mislabelled data points in the training data of class 3.
These mislabelled data points lead the trained classifier to recognized the region of class 3 as a part of class 1 or class 2.
Figure \ref{Fig_KernelM1} shows that the trained classifier by the chance constrained CS-SVM is more robust and is less influenced by mislabelled data points.

\subsection{UCI data sets}

In this experiment, the performance of chance constrained CS-SVM was demonstrated on the popular UCI data set \cite{UCIdataset}.
%{\color{blue}
Furthermore, we compare the performance of our models with
the robust classification model (RC) for classification with feature and label uncertainty proposed in \cite{Robust_classification}, by combining the kernel technique.
To apply the method in \cite{Robust_classification},
the multiclass data sets were converted into binary data sets using the one-versus-all scheme \cite{Multi_to_binary}.
%}
%{\color{blue}
Three data sets have been selected from the repository, namely the Banknote Authentication Data Set (BKA, \cite{Banknote}) for binary classification, Iris Data Set (IRIS, \cite{Iris}) and Seeds Data Set (SEED, \cite{Seed}) for multiclass classification.
%}
BKA data set were extracted from images that were taken for the evaluation of an authentication procedure for banknotes.
IRIS data set contains 3 classes of iris plant:
Each class refers to a type of iris plant, where one class is linearly separable from the other 2 and the latter are not linearly separable from each other.
%{\color{blue}
SEED data set involves the prediction of species given measurements of seeds from three varieties of wheat.
%}
The BKA data set contains 1372 data points, where there are 762 data points for class 1 and 610 data points for class 2.
And 50 data points were selected randomly for each class as the training data, the rest as testing data.
IRIS data set contains 150 data points, where there are 50 data points for each class.
Among them, 20 data points were selected randomly for each class as the training data, the rest points as testing data.
%{\color{blue}
SEED data set contains 210 data points.
And there are 70 data points for each class.
For each class, 30 data points were selected randomly as the training data, the rest as testing data.
%}
Linear, polynomial and RBF kernels were all tested in both of these data sets.
This process was replicated 10 times as the random selection of training and evaluation sets can impact performance.

The classification result on BKA data set are summarized in Table \ref{Table_BKA}.
Table \ref{Table_BKA} presents the mean accuracy(standard derivation) of classification for BKA data set with different kernels from the second column to the sixth column, and the three model names in the first column.
We can observe that the chance constrained CS-SVM can always perform better than the CS-SVM with soft margin and
the robust classification model RC, especially when polynomial kernel with $d \geq 4$ and RBF kernel are applied.
Although the mean accuracy of the classification by chance constrained CS-SVM is not be better than the mean accuracy of CS-SVM with soft margin,
the standard derivation of accuracy by chance constrained CS-SVM is always smaller, compared with CS-SVM with soft margin.
Considering the total number of data set, the size of training data is relatively small.
As shown in Table \ref{Table_BKA}, the stable accuracy of trained classifier by chance constrained CS-SVM implies the potential ability of chance constrained CS-SVM to handle the classification with small samples.
%{\color{blue}
Compared with the robust classification model RC, it is not hard to notice that the mean accuracy of the chance constrained CS-SVM is just slightly better than the mean accuracy of the model RC.
At the same time, the standard derivation of model RC is always smaller.
This is because the robust model always considers the worst case, which may lead the trained classifier to be too conservative that accuracy can be sacrificed.
%}

\begin{table}[htbp]
\begin{center}
\begin{minipage}{\textwidth}
\caption{Classification results for BKA}\label{Table_BKA}
%\label{Tab_first}
\begin{tabular*}{\textwidth}{@{\extracolsep{\fill}}lllllll@{\extracolsep{\fill}}}
\toprule
           \multirow{2}{*}{Model} & \multicolumn{5}{c}{Kernels}    \\ \cmidrule(lr){2-6}
                  & \multicolumn{1}{c}{LIN} & \multicolumn{1}{c}{$d=2$} & \multicolumn{1}{c}{$d=4$}& \multicolumn{1}{c}{$d=6$} & \multicolumn{1}{c}{RBF} \\
\midrule
 CS             & $98.2\%(3.9\%)$  &$\bm{96.3\%}(5.4\%)$  & $88.7\%(8.1\%)$  & $78.4\%(10.6\%)$  &  $94.9\%(5.9\%)$       \\
 RC                & $98.2\%(3.0\%)$  & $95.9\%(3.8\%)$  & $93.1\%(5.2\%)$  & $88.3\%(6.9\%)$  &  $95.4\%(3.3\%)$     \\
 CC                    & $\bm{98.3\%}(3.4\%)$  & $96.2\%(4.3\%)$  &$\bm{93.8\%}(6.3\%)$  &$\bm{89.6\%}(8.4\%)$  &  $\bm{95.7\%}(3.6\%)$     \\
\botrule
\end{tabular*}
\end{minipage}
\end{center}
\end{table}

Table \ref{Table_IRIS} presents the classification results for IRIS data set.
The mean accuracy(standard derivation) of classification for IRIS data set with different kernels are listed from the second column to the last column, respectively.
The mean accuracy and the associated standard derivation when polynomial kernel has $d=6$ are absent because of the sparsity of kernel matrix, which does not allow to properly solve the optimization problem.
Regardless the performance with polynomial kernel with $d=6$, it is not hard to notice that, in contrast to the performance on BKA data set for binary classification,
%{\color{blue}
the CS-SVM with soft margin always performs far behind to both chance constrained CS-SVM and model RC on IRIS data set.
Compared with the CS-SVM with soft margin,
not only the standard derivation of the accuracy by chance constrained CS-SVM is always smaller, but also the mean accuracy of the classification by chance constrained CS-SVM is significantly larger.
The best mean accuracy was obtained when RBF kernel and polynomial kernel with $d=2$ were applied in the chance constrained CS-SVM.
Similar with the performance on BKA data set, the mean accuracy of model RC is slightly smaller than the mean accuracy of chance constrained CS-SVM by about $1\% \sim 2\%$, while the standard derivation of model RC is always smaller than the standard derivation of chance constrained CS-SVM by about $ 0.2\% \sim 0.4\%$, due to the sacrifice in accuracy as previously explained.

%}

\begin{table}[htbp]
\begin{center}
\begin{minipage}{\textwidth}
\caption{Classification results for IRIS}\label{Table_IRIS}
\begin{tabular*}{\textwidth}{@{\extracolsep{\fill}}lllllll@{\extracolsep{\fill}}}
\toprule
           \multirow{2}{*}{Model} & \multicolumn{5}{c}{Kernels}    \\ \cmidrule(lr){2-6}
                  & \multicolumn{1}{c}{LIN} & \multicolumn{1}{c}{$d=2$} & \multicolumn{1}{c}{$d=4$}& \multicolumn{1}{c}{$d=6$} & \multicolumn{1}{c}{RBF} \\
\midrule
 CS         & $85.2\%(2.9\%)$  & $88.4\%(3.7\%)$  & $84.9\%(5.6\%)$  & -  &  $92.3\%(2.9\%)$    \\
 RC                & $93.4\%(1.3\%)$  & $95.1\%(2.1\%)$  & $94.6\%(2.8\%)$  & -  &  $96.9\%(1.4\%)$     \\
 CC                & $\bm{96.4\%}(1.5\%)$  & $\bm{97.2\%}(2.4\%)$  & $\bm{96.7\%}(3.2\%)$  & -  &  $\bm{97.8\%}(1.6\%)$    \\
\botrule
\end{tabular*}
\end{minipage}
\end{center}
\end{table}

%{\color{blue}
Table \ref{Table_SEED} presents the classification results for SEED data set.
The mean accuracy(standard derivation) of classification for SEED data set with different kernels are listed from the second column to the last column, respectively.
From Table \ref{Table_SEED}, we can observe that on the SEED data set, the difference between performances of CS-SVM with soft margin, model RC and chance constrained CS-SVM are not as significant as the performances on IRIS data set.
In spite of this, the chance constrained CS-SVM provided the highest mean accuracy $98.6\%$, when polynomial kernel with $d=6$ were applied.
The standard derivation of chance constrained model is slightly larger than the standard derivation of model RC by about $ 0.1\% \sim 0.5\% $, but much smaller than the standard derivation of CS-SVM with soft margin.

\begin{table}[htbp]
\begin{center}
\begin{minipage}{\textwidth}
\caption{Classification results for SEED}\label{Table_SEED}
%\label{Tab_first}
\begin{tabular*}{\textwidth}{@{\extracolsep{\fill}}lllllll@{\extracolsep{\fill}}}
\toprule
           \multirow{2}{*}{Model} & \multicolumn{5}{c}{Kernels}    \\ \cmidrule(lr){2-6}
                  & \multicolumn{1}{c}{LIN} & \multicolumn{1}{c}{$d=2$} & \multicolumn{1}{c}{$d=4$}& \multicolumn{1}{c}{$d=6$} & \multicolumn{1}{c}{RBF} \\
\midrule
 CS     & $91.9\%(2.3\%)$  & $\bm{95.2\%}(4.2\%)$  & $96.1\%(4.8\%)$  & $96.8\%(5.4\%)$  &  $93.7\%(4.4\%)$ \\
 RC           & $92.0\%(1.1\%)$  & $94.9\%(1.3\%)$  & $96.6\%(1.8\%)$  & $97.5\%(2.1\%)$  &  $95.1\%(2.5\%)$ \\
 CC           & $\bm{92.4\%}(1.2\%)$  & $94.8\%(1.7\%)$  & $\bm{97.1\%}(2.1\%)$  & $\bm{98.6\%}(2.4\%)$  &  $\bm{95.7\%}(3.0\%)$ \\
\botrule
\end{tabular*}
\end{minipage}
\end{center}
\end{table}

From the above experimental results on different data sets, it is not hard to conclude that on these data sets, the chance constrained conic-segmentation support vector machine \eqref{CC_SVM_Integer} always performs best, which can provide best mean accuracy and small standard derivation.
The robust classification model proposed in \cite{Robust_classification} can also provide good accuracy with smallest standard derivation.
The conic-segmentation support vector machine with soft margin \eqref{CC_SVM_CS-Soft} performs worst with worst accuracy and large standard derivation in general.
The experiment shows that the chance constrained conic-segmentation support vector machine can achieve effectiveness and robustness in both binary classification and multiclassification problems.
%}

\section{Conclusion}

In this paper, a chance constrained conic-segmentation support vector machine model has been proposed, which can be seen as an extension of the conic-segmentation support vector machine with uncertain or mislabelled data.
This model can ensure a small probability of misclassification for the uncertain data.
Based on a data set, the chance constrained CS-SVM can be trained by solving a mixed integer programming problem.
To handle the nonlinear classification, a corresponding kernelization model has also been derived.
In addition, geometric illustration has been presented to show how the chance constrained CS-SVM works.
The chance constrained CS-SVM has also been experimentally compared to CS-SVM with soft margin on both artificial data and real data in both binary classification and multiclass classification.
The experimental results demonstrate that chance constrained CS-SVM is both effective and robust.

For future research, the numerical algorithms on big data is a potential direction and application.
Currently,
as the chance constrained CS-SVM model is obtained by solving a mixed integer programming problem, this could lead to large solving times, especially on a big data set.
In the future, through proposed approaches, more meaningful results could be obtained on some big data sets.

\appendix
\section{Appendix}
\subsection{Proof of Theorem \ref{TH_Kernel}}\label{AppendixA}
\begin{proof}
To get the optimal classifier in feature space, the problem \eqref{CC_SVM_Integer_K} should be solved.
From Lemma \ref{lem_span}, problem \eqref{CC_SVM_Integer_K} can be reformulated as
\begin{equation}
\label{CC_SVM_Integer_KRe}
\begin{aligned}
\min\limits_{\omega_j,\bm{b}}& ~\frac{1}{2}\sum_{j=1}^{d_T}\|\omega_j\|^2 \\
\mathrm{s.t.}~&~ \sum_{j=1}^{d_T}(\bm{v}^n_{s,t})_j\left(( \omega_j )^\top\bm{\phi}(\bm{x}_i) + \bm{b}_j - (\bm{u}_{n,s})_j \right) + M_s z^s_i \geq 0,\\
& ~~~~~~~~~~~~~~~~~~ t\ne s, t \in \mathbb{Z}_n, \bm{x}_i \in \Theta_s, i \in \mathbb{Z}_{N_s}, s \in \mathbb{Z}_n,\\
& ~\sum_{ i \in \mathbb{Z}_{N_s}} z^s_i \leq \alpha_sN_s, s \in \mathbb{Z}_n,\\
&~ z^s_i \in \{0,1\}, i \in \mathbb{Z}_{N_s}, s \in \mathbb{Z}_n,
\end{aligned}
\end{equation}
Without loss of generality, for $j=1,\cdots, d_T$, any optimal $\omega_j$ can be written as
\begin{equation}\label{Combine}
  \omega_j = \sum_{s \in \mathbb{Z}_n}\sum_{\bm{x}_i \in \Theta_s}\gamma_{si}^j\bm{\phi}(\bm{x}_i).
\end{equation}
This equation could be derived in a more formal way by using the representer theorem \cite{RepresenterTh}.

By substituting expression \eqref{Combine} for $\omega_j, j=1,\cdots,d_T$, in problem \eqref{CC_SVM_Integer_KRe}, we have
\[
\begin{array}{l}
  \|\omega_j\|^2 = (\gamma^j)^\top K \gamma^j, \\
  ( \omega_j )^\top\bm{\phi}(\bm{x}_i) = \sum_{s \in \mathbb{Z}_n}\sum_{\bm{x}_l \in \Theta_s}\gamma_{sl}^jK_{il} = K_i^\top\gamma^j,
\end{array}
\]
where $K$ is a matrix with elements $K_{ij} = \bm{\phi}(\bm{x}_i)^\top\bm{\phi}(\bm{x}_j), \bm{x}_i \in \Theta_s, \bm{x}_j \in \Theta_{s'}, s,s' \in \mathbb{Z}_n $,
and $ K_i $ is a vector with elements $ K_{ij} = \bm{\phi}(\bm{x}_i)^\top\bm{\phi}(\bm{x}_j), \bm{x}_j \in \Theta_{s}, s \in \mathbb{Z}_n $.
Then, problem \eqref{CC_SVM_Integer_KRe} can be reformulated as
\begin{equation*}
%\label{CC_SVM_Kernel}
\begin{aligned}
\min\limits_{\gamma_j,\bm{b}}& ~\frac{1}{2}\sum_{j=1}^{d_T}(\gamma^j)^\top K \gamma^j \\
\mathrm{s.t.}~&~ \sum_{j=1}^{d_T}(\bm{v}^n_{s,t})_j\left(K_i^\top\gamma^j + \bm{b}_j - (\bm{u}_{n,s})_j \right) + M_s z^s_i \geq 0,\\
& ~~~~~~~~~~~~~~~~~~ t\ne s, t \in \mathbb{Z}_n, \bm{x}_i \in \Theta_s, i \in \mathbb{Z}_{N_s}, s \in \mathbb{Z}_n,\\
& ~\sum_{ i \in \mathbb{Z}_{N_s}} z^s_i \leq \alpha_sN_s, s \in \mathbb{Z}_n,\\
&~ z^s_i \in \{0,1\}, i \in \mathbb{Z}_{N_s}, s \in \mathbb{Z}_n.
\end{aligned}
\end{equation*}
From Lemma \ref{lem_span}, it is not hard to observe that the classifier $ \bm{g}(\bm{\phi}(\bm{x})) $ can be written in the form of \eqref{machine_K},
which concludes the theorem.
\end{proof}

\section*{Declarations}
%The authors have no relevant financial or non-financial interests to disclose.
{\bf Conflict of Interest}: The authors declare that they have no conflict of interest.

\noindent
{\bf Data availability statements}: The generated data of this study are available on request from the corresponding author SP. The UCI datasets analysed during the current study are available in the UCI Machine Learning Repository [http://archive.ics.uci.edu/ml] \cite{UCIdataset}.

%\begin{itemize}
%\item Conflict of interest: The authors declare that they have no conflict of interest.
%\item Data availability statements: The generated data of this study are available on request from the corresponding author SP.
%    The UCI datasets analysed during the current study are available in the UCI Machine Learning Repository [http://archive.ics.uci.edu/ml] \cite{UCIdataset}.
%\end{itemize}

\bibliography{sn-bibliography}% common bib file


\end{document}